\documentclass[11pt,a4paper]{article}
\usepackage[margin=1.25in]{geometry}

\usepackage{url}
\usepackage{lmodern}
\usepackage[english]{babel}
\usepackage{ucs}
\usepackage[utf8x]{inputenc}
\usepackage[textsize=footnotesize,color=green!40]{todonotes}
\usepackage{amsfonts}
\usepackage{amsmath}
\usepackage{amssymb}
\usepackage{amsthm}
\usepackage{mathtools}
\usepackage{natbib}

\usepackage{algorithm}
\usepackage{algorithmic}

\usepackage{stmaryrd}

\newcommand{\R}{\mathbb{R}}

\newcommand{\sT}{\mathbb{T}}

\renewcommand{\vec}[1]{\ensuremath{\mathbf{#1}}}
\newcommand{\vecs}[1]{\ensuremath{\mathbf{\boldsymbol{#1}}}}
\newcommand{\mat}[1]{\ensuremath{\mathbf{#1}}}
\newcommand{\mats}[1]{\ensuremath{\mathbf{\boldsymbol{#1}}}}
\newcommand{\ten}[1]{\mat{\ensuremath{\boldsymbol{\mathcal{#1}}}}}
\newcommand{\idx}[1]{(#1)}

\renewcommand{\v}{\vec{v}}
\newcommand{\M}{\mat{M}}
\newcommand{\T}{\ten{T}}
\newcommand{\I}{\mat{I}}
\renewcommand{\H}{\mat{H}}
\newcommand{\G}{\mat{G}}
\renewcommand{\P}{\mat{P}}
\renewcommand{\S}{\mat{S}}
\newcommand{\Q}{\mat{Q}}
\newcommand{\U}{\mat{U}}
\newcommand{\D}{\mat{D}}
\newcommand{\V}{\mat{V}}
\newcommand{\Proj}{\mats{\Pi}}

\newcommand{\Ts}{\mathfrak{T}}
\newcommand{\Cs}{\mathfrak{C}}

\newcommand{\TF}{\Ts}
\newcommand{\CF}{\Cs}
\newcommand{\depth}[1]{\mathrm{depth}(#1)}
\newcommand{\size}[1]{\mathrm{size}(#1)}
\newcommand{\drop}[1]{\mathrm{drop}(#1)}
\newcommand{\yield}[1]{\langle #1 \rangle}
\newcommand{\sstar}{\Sigma^\star}
\newcommand{\gap}{\ast}
\newcommand{\ctm}{\mats{\Xi}}
\newcommand{\X}{\mathfrak{X}}
\newcommand{\weight}{\mathrm{weight}}
\newcommand{\Bs}{\mathfrak{B}}
\newcommand{\wzeroB}{\vecs{\beta}}
\newcommand{\winfB}{\vecs{\psi}}

\newcommand{\wzero}{\vecs{\alpha}}
\newcommand{\winf}{\vecs{\omega}}
\newcommand{\wtaSigma}{\langle \wzero, \T, \{\winf_\sigma\}_{\sigma \in \Sigma} \rangle}
\newcommand{\wta}{\langle \wzero, \T, \{\winf_\sigma\} \rangle} 
\newcommand{\wtaQ}{\langle \Q^\top \wzero , \T(\Q^{-\top},\Q,\Q), \{\Q^{-1} \winf_\sigma\} \rangle}
\newcommand{\gramc}{\G_\Cs}
\newcommand{\gramcapprox}{\hat{\G}_\Cs}
\newcommand{\gramt}{\G_\Ts}
\newcommand{\gramtapprox}{\hat{\G}_\Ts}

\newtheorem*{theorem*}{Theorem}
\newtheorem{theorem}{Theorem}
\newtheorem{proposition}{Proposition}
\newtheorem{lemma}{Lemma}

\DeclareMathOperator*{\vectorize}{vec}
\DeclareMathOperator*{\reshape}{reshape}
\newcommand{\eqdef}{\triangleq}

\newcommand{\rank}{\mathop{rank}}
\newcommand{\norm}[1]{\|#1\|}
\newcommand{\kron}{\otimes}
\newcommand{\Ker}{\mathrm{Ker}}
\renewcommand{\Im}{\mathrm{Im}}
\newcommand{\bigo}[1]{\mathcal{O}(#1)}
\newcommand{\singv}{\mathfrak{s}}

\renewcommand{\cite}{\citep}

\begin{document}

\title{Low-Rank Approximation of Weighted Tree Automata}

\author{Guillaume Rabusseau\footnote{Contact author: \texttt{guillaume.rabusseau@lif.univ-mrs.fr}.} \\ Aix-Marseille University \and Borja Balle \\ Lancaster University \and Shay B. Cohen \\ University of Edinburgh }

\maketitle

\begin{abstract}
We describe a technique to minimize weighted tree automata (WTA), a powerful formalism that subsumes
probabilistic context-free grammars (PCFGs) and latent-variable PCFGs. Our method relies on
a singular value decomposition of the underlying Hankel matrix defined by the WTA.
Our main theoretical result is an efficient algorithm for computing the SVD of
an infinite Hankel matrix implicitly represented as a WTA.
We provide an analysis of the approximation error induced by the minimization, and
we evaluate our method on real-world data originating in newswire treebank. We show that the model
achieves lower perplexity than previous methods for PCFG minimization, and also
is much more stable due to the absence of local optima.

\end{abstract}

\section{Introduction}\label{sec:introduction}

Probabilistic context-free grammars (PCFG) provide a powerful statistical
formalism for modeling important phenomena occurring in natural language.  In
fact, learning and parsing algorithms for PCFG are now standard tools in natural
language processing pipelines.
Most of these algorithms can be naturally extended to the superclass of weighted
context-free grammars (WCFG), and closely related models like weighted tree
automata (WTA) and latent probabilistic context-free grammars (LPCFG).
The complexity of these algorithms depends on the size of the grammar/automaton, typically controlled by the number of rules/states. Being able to control this complexity is essential in operations like parsing, which is typically executed every time the model is used to make a prediction.
In this paper we present an algorithm that given a WTA with $n$ states and a
target number of states $\hat{n} < n$, returns a WTA with $\hat{n}$ states that
is a good approximation of the original automaton. This can be interpreted as a
low-rank approximation method for WTA through the direct connection between number of
states of a WTA and the rank of its associated Hankel matrix. This opens the
door to reducing the complexity of algorithms working with WTA at the price of
incurring a small, controlled amount of error in the output of such algorithms.

Our techniques are inspired by recent developments in spectral learning
algorithms for different classes of models on sequences
\cite{hsu2012spectral,denis,psr,mlj13spectral} and trees
\cite{bailly2010spectral,cohen-14b}, and subsequent investigations into low-rank
spectral learning for predictive state representations
\cite{kulesza2014low,kulesza2015low} and approximate minimization of weighted
automata \cite{bpp15}.
In spectral learning algorithms, data is used to reconstruct a finite block of a
Hankel matrix and an SVD of such matrix then reveals a low-dimensional space
where a linear regression recovers the parameters of the model.
In contrast, our approach computes the SVD of the
\emph{infinite} Hankel matrix associated with a WTA. Our main result is an
efficient algorithm for computing this singular value decomposition by operating
directly on the WTA representation of the Hankel matrix; that is, without the need to explicitly represent this infinite matrix at any point.
Section~\ref{sec:approxmin} presents the main ideas underlying our approach.
Add a comment to this line
An efficient algorithmic implementation of these ideas is discussed in
Section~\ref{sec:algorithm}, and a theoretical analysis of the approximation error
induced by our minimization method is given in Section~\ref{sec:bounds}.
Proofs of all results stated in the paper can be found in appendix.

The idea of speeding up parsing with (L)PCFG by approximating the original model
with a smaller one was recently studied in \cite{cohen-12c,cohen-13a}, where a
tensor decomposition technique was used in order to obtain the minimized model.
We compare that approach to ours in the experiments presented in
Section~\ref{sec:experiments}, where both techniques are used to compute
approximations to a grammar learned from a corpus of real linguistic data.
It was observed in \cite{cohen-12c,cohen-13a} that a side-effect of reducing the
size of a grammar learned from data was a slight improvement in parsing performance. The
number of parameters in the approximate models is smaller, and as such,
generalization improves. We show in our experimental section that our
minimization algorithms have the same effect in certain parsing scenarios. In
addition, our approach yields models which give lower perplexity on an unseen
set of sentences, and provides a better approximation to the original model in
terms of $\ell_2$ distance.
It is important to remark that in contrast with the tensor decompositions
in \cite{cohen-12c,cohen-13a} which are susceptible to local optima problems,
our approach resembles a power-method approach to SVD, which yields efficient globally convergent algorithms.
Overall, we observe in our experiments that this renders a more stable
minimization method than the one using tensor decompositions.

\subsection{Notation}\label{sec:notation}

For an integer $n$, we write $[n] = \{1,\ldots,n\}$.
We use lower case bold letters (or symbols) for vectors (e.g.\ $\v \in \R^{d_1}$),
upper case bold letters for matrices (e.g.\ $\M \in \R^{d_1 \times d_2}$) and
bold calligraphic letters for third order tensors (e.g.\ $\T \in \R^{d_1
\times d_2 \times d_3}$).
Unless explicitly stated, vectors are by default column vectors. The identity
matrix will be written as $\I$.
Given $i_1 \in [d_1], i_2 \in [d_2], i_3 \in [d_3]$ we use $\v(i_1)$,
$\M(i_1,i_2)$, and $\T(i_1,i_2,i_3)$ to denote the corresponding
entries.
The $i$th row (resp. column) of a matrix $\M$ will be noted
$\M\idx{i,:}$ (resp. $\M\idx{:,i}$). This notation is extended to
slices across the three modes of a tensor in the straightforward way.
If $\v \in \R^{d_1}$ and $\v' \in \R^{d_2}$, we use $\v \kron \v' \in \R^{d_1
\cdot d_2}$ to denote the Kronecker product between vectors, and its
straightforward extension to matrices and tensors.
Given a matrix $\M \in \R^{d_1 \times d_2}$ we use $\vectorize(\M) \in \R^{d_1
\cdot d_2}$ to denote the column vector obtained by concatenating the columns of
$\M$.
Given a tensor $\T \in \R^{d_1 \times d_2 \times d_3}$ and matrices $\M_i \in
\R^{d_i \times d'_i}$ for $i \in [3]$, we define a tensor $\T(\M_1,\M_2,\M_3)
\in \R^{d'_1 \times d'_2 \times d'_3}$ whose entries are given by
\begin{equation*}
\T(\M_1, \M_2, \M_3)(i_1,i_2,i_3) = 
\sum_{j_1,j_2,j_3} \T(j_1,j_2,j_3)
\M_1(j_1,i_1) \M_2(j_2, i_2) \M_3(j_3, i_3) \enspace.
\end{equation*}
This operation corresponds to contracting $\T$ with $\M_i$ across the $i$th mode
of the tensor for each $i$.

\section{Approximate Minimization of WTA and SVD of Hankel Matrices}
\label{sec:approxmin}

In this section we present the first contribution of the paper. Namely, the
existence of a canonical form for weighted tree automata inducing the
singular value decomposition of the infinite Hankel matrix associated with the
automaton.
We start by recalling several definitions and well-known facts about WTA that
will be used in the rest of the paper. Then we proceed to establish the
existence of the canonical form, which we call the \emph{singular value tree
automaton}.
Finally we indicate how removing the states in this canonical form that
correspond to the smallest singular values of the Hankel matrix leads to an
effective procedure for model reduction in WTA.

\subsection{Weighted Tree Automata}

Let $\Sigma$ be a finite alphabet. We use $\sstar$ to denote the set of all
finite strings with symbols in $\Sigma$ with $\lambda$ denoting the empty
string. We write $|x|$ to denote the length of a string $x \in \sstar$. The
number of occurences of a symbol $\sigma \in \Sigma$ in a string $x \in \sstar$
is denoted by $|x|_\sigma$. 

The set of all \emph{rooted full binary trees} with leafs in $\Sigma$ is the
smallest set $\Ts_\Sigma$ such that $\Sigma \subset \Ts_\Sigma$ and $(t_1,t_2)
\in \Ts_\Sigma$ for any $t_1, t_2 \in \Ts_\Sigma$. We shall just write $\Ts$
when the alphabet $\Sigma$ is clear from the context.
The \emph{size} of a tree $t \in \Ts$ is denoted by $\size{t}$ and defined
recursively by $\size{\sigma} = 0$ for $\sigma \in \Sigma$, and
$\size{(t_1,t_2)} = \size{t_1} + \size{t_2} + 1$; that is, the number of
internal nodes in the tree.
The \emph{depth} of a tree $t \in \Ts$ is denoted by $\depth{t}$ and defined
recursively by $\depth{\sigma} = 0$ for $\sigma \in \Sigma$, and
$\depth{(t_1,t_2)} = \max\{\depth{t_1},\depth{t_2}\} + 1$; that is, the
distance from the root of the tree to the farthest leaf.
The \emph{yield} of a tree $t \in \Ts$ is a string $\yield{t} \in \Sigma^*$
defined as the left-to-right concatenation of the symbols in the leafs of $t$,
and can be recursively defined by $\yield{\sigma} = \sigma$, and
$\yield{(t_1,t_2)} = \yield{t_1} \cdot \yield{t_2}$.
The total number of nodes (internal plus leafs) of a tree $t$ is denoted by
$|t|$ and satisfies $|t| = \size{t} + |\yield{t}|$.

Let $\Sigma' = \Sigma \cup \{\gap\}$, where $\gap$ is a symbol \emph{not} in
$\Sigma$.
The set of \emph{rooted full binary context trees} is the set $\Cs_\Sigma = \{ c
\in \Ts_{\Sigma'} \;|\; |\yield{c}|_{\gap} = 1 \}$; that is, a context
$c \in \Cs_\Sigma$ is a tree in $\Ts_{\Sigma'}$ in which the symbol $\gap$
occurs exactly in one leaf.
Note that because given a context $c = (t_1,t_2) \in \Cs_\Sigma$ with $t_1,
t_2 \in \Ts_{\Sigma'}$ the symbol $\gap$ can only appear in one of the $t_1$ and
$t_2$, we must actually have $c = (c',t)$ or $c = (t,c')$ with $c' \in
\Cs_\Sigma$ and $t \in \Ts_\Sigma$.
The \emph{drop} of a context $c \in \Cs$ is the distance between the root and
the leaf labeled with $\gap$ in $c$, which can be defined recursively as
$\drop{\gap} = 0$, $\drop{(c,t)} = \drop{(t,c)} = \drop{c} + 1$.

We usually think as the leaf with the symbol $\gap$ in a context as a
placeholder where the root of another tree or another context can be inserted.
Accordingly, given $t \in \Ts$ and $c \in \Cs$, we can define $c[t] \in \Ts$ as the
tree obtained by replacing the occurence of $\gap$ in $c$ with $t$. Similarly,
given $c, c' \in \Cs$ we can obtain a new context tree $c[c']$ by replacing the
occurence of $\gap$ in $c$ with $c'$.
See Figure~\ref{fig:trees_contexts} for some illustrative examples.

\begin{figure}
\begin{center}
\includegraphics[scale=1]{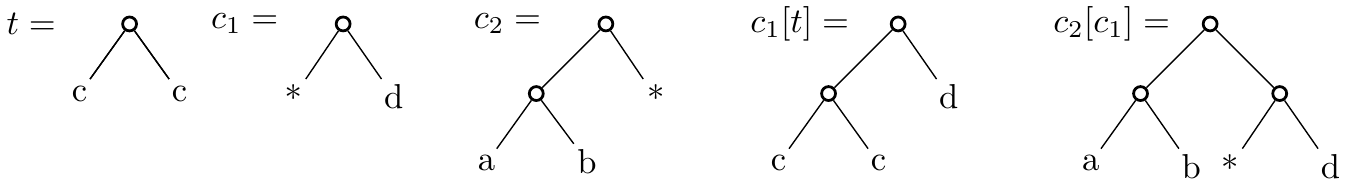}
\end{center}
\caption{Examples of trees ($t,c_1[t] \in \Ts_\Sigma$) and contexts 
($c_1,c_2,c_2[c_1] \in \Cs_\Sigma$) on the alphabet $\Sigma=\{a,b,c,d\}$. In our
notation:
$c_1[t] = ((c,c),d)$,
$\size{c_1[t]} = 2$,
$\depth{c_1[t]} = 2$,
$\yield{t} = cc$,
$\drop{c_2[c_1]} = 2$
}
\label{fig:trees_contexts}
\end{figure}

A \emph{weighted tree automaton} (WTA) over $\Sigma$ is a tuple $A = \wtaSigma$,
where $\wzero \in \R^n$ is the vector of \emph{initial weights}, $\T \in
\R^{n\times n \times n}$ is the tensor of \emph{transition weights}, and
$\winf_\sigma \in \R^n$ is the vector of \emph{terminal weights} associated with
$\sigma \in \Sigma$. The dimension $n$ is the number of states of the automaton,
which we shall sometimes denote by $|A|$.
A WTA $A = \wta$ \emph{computes} a function $f_A : \Ts_\Sigma \to \R$ assigning
to each tree $t \in \Ts$ the number computed as $f_A(t) = \wzero^\top
\winf_A(t)$, where $\winf_A(t) \in \R^n$ is obtained recursively as
$\winf_A(\sigma) = \winf_\sigma$, and $\winf_A((t_1,t_2)) =
\T(\I,\winf_A(t_1),\winf_A(t_2))$ --- note the matching of dimensions in this
last expression since contracting a third order tensor with a matrix in the
first mode and vectors in the second and third mode yields a vector. In many
cases we shall just write $\winf(t)$ when the automaton $A$ is clear from the
context. While WTA are usually defined over arbitrary ranked trees, only
considering binary trees does not lead to any loss of generality since WTA on ranked
trees are equivalent to WTA on binary trees (see \cite{bailly2010spectral} for
references). Additionally, one could consider binary trees where each internal
node is labelled, which leads to the definition of WTA with multiple transition
tensors. Our results can be extended to this case without much effort, but we
state them just for WTA with only one transition tensor to keep the notation
manageable.

An important observation is that there exist more than one WTA computing the
same function --- actually there exist infinitely many. An important
construction along these lines is the \emph{conjugate} of a WTA $A$ with $n$
states by an invertible matrix $\Q \in \R^{n \times n}$. If $A = \wta$, its
conjugate by $\Q$ is $A^\Q = \wtaQ$, where $\Q^{-\top}$ denotes the inverse of
$\Q^\top$.
To show that $f_A = f_{A^\Q}$ one applies an induction argument on $\depth{t}$
to show that for every $t \in \Ts$ one has $\winf_{A^\Q}(t) = \Q^{-1}
\winf_A(t)$. The claim is obvious for trees of zero depth $\sigma \in \Sigma$,
and for $t = (t_1,t_2)$ one has
\begin{align*}
 \winf_{A^\Q}((t_1,t_2)) &=
(\T(\Q^{-\top},\Q,\Q))(\I,\winf_{A^\Q}(t_1),\winf_{A^\Q}(t_2)) \\
& = (\T(\Q^{-\top},\Q,\Q))(\I,\Q^{-1} \winf_{A}(t_1),Q^{-1} \winf_{A}(t_2)) \\
& =
\T(\Q^{-\top},\winf_A(t_1),\winf_A(t_2)) \\
& = \Q^{-1} \T(\I,\winf_A(t_1),\winf_A(t_2))
\enspace,
\end{align*}
where we just used some simple rules of tensor algebra.

An arbitrary function $f : \Ts \to \R$ is called \emph{rational} if there exists
a WTA $A$ such that $f = f_A$. The number of states of the smallest such WTA is
the \emph{rank} of $f$ --- we shall set $\rank(f) = \infty$ if $f$ is not
rational. A WTA $A$ with $f_A = f$ and $|A| = \rank(f)$ is called
\emph{minimal}.
Given any $f : \Ts \to \R$ we define its \emph{Hankel matrix} as the infinite
matrix $\H_f \in \R^{\Cs \times \Ts}$ with rows indexed by contexts, columns
indexed by trees, and whose entries are given by $\H_f(c,t) = f(c[t])$.
Note that given a tree $t' \in \Ts$ there are exactly $|t'|$ different ways of
splitting $t' = c[t]$ with $c \in \Cs$ and $t \in \Ts$.
This implies that $\H_f$ is a highly redundant representation for $f$, and it
turns out that this redundancy is the key to proving the following fundamental
result about rational tree functions.

\begin{theorem}[\cite{Bozapalidis_Louscou-Bozapalidou_1983}]\label{thm:rank}
For any $f : \Ts \to \R$ we have $\rank(f) = \rank(\H_f)$.
\end{theorem}

\subsection{Rank Factorizations of Hankel Matrices}

The theorem above can be rephrased as saying that the rank of $\H_f$ is finite
if and only if $f$ is rational.
When the rank of $\H_f$ is indeed finite --- say $\rank(\H_f) = n$ --- one can
find two rank $n$ matrices $\P \in \R^{\Cs
\times n}$, $\S \in \R^{n \times \Ts}$ such that $\H_f = \P \S$. In this case we
say that $\P$ and $\S$ give a \emph{rank factorization} of $\H_f$.
We shall now refine Theorem~\ref{thm:rank} by showing that when $f$ is rational,
the set of all possible rank factorizations of $\H_f$ is in direct correpondance
with the set of minimal WTA computing $f$.

The first step is to show that any minimal WTA $A = \wta$ computing $f$ induces
a rank factorization $\H_f = \P_A \S_A$. We build $\S_A \in \R^{n \times \Ts}$ by
setting the column corresponding to a tree $t$ to $\S_A(:,t) = \winf_A(t)$. In
order to define $\P_A$ we need to introduce a new mapping $\ctm_A : \Cs \to
\R^{n \times n}$ assigning a matrix to every context as follows: $\ctm_A(\gap) =
\I$, $\ctm_A((c,t)) = \T(\I,\ctm_A(c),\winf_A(t))$, and $\ctm_A((t,c)) =
\T(\I,\winf_A(t),\ctm_A(c))$. If we now define $\wzero_A : \Cs \to \R^n$ as
$\wzero_A(c)^\top = \wzero^\top \ctm_A(c)$, we can set the row of $\P_A$
corresponding to $c$ to be $\P_A(c,:) = \wzero_A(c)^\top$.
With these definitions one can easily show by induction on $\drop{c}$ that
$\ctm_A(c) \winf_A(t) = \winf_A(c[t])$ for any $c \in \Cs$ and $t \in \Ts$.
Then it is immediate to check that $\H_f = \P_A \S_A$:
\begin{align}
 \sum_{i=1}^n \P_A(c,i) \S_A(i,t) &= \wzero_A(c)^\top \winf_A(t) = \wzero^\top \ctm_A(c) \winf_A(t) \notag \\
 &= \wzero^\top \winf_A(c[t]) 
 = f_A(c[t]) \notag \\
 &= \H_f(c,t) \label{eqn:fct} \enspace.
\end{align}
As before, we shall sometimes just write $\ctm(c)$ and $\wzero(c)$ when $A$ is
clear from the context. We can now state the main result of this section, which
generalizes similar results in \cite{bpp15,mlj13spectral} for weighted automata
on strings.

\begin{theorem}\label{thm:rankfact}
Let $f : \Ts \to \R$ be rational. If $\H_f = \P \S$ is a rank factorization,
then there exists a minimal WTA $A$ computing $f$ such that $\P_A = \P$ and
$\S_A = \S$.
\end{theorem}
\begin{proof}
See Appendix~\ref{proof:thm:rankfact}.
\end{proof}

\subsection{Approximate Minimization with the Singular Value Tree Automaton}

Equation \eqref{eqn:fct} can be interpreted as saying that given a fixed factorization $\H_f = \P_A \S_A$, the value $f_A(c[t])$ is given by the inner product $
\sum_{i} \wzero_A(c)(i) \winf_A(t)(i)$. Thus, $\wzero_A(c)(i)$ and $\winf_A(t)(i)$ quantify the influence of state $i$ in the computation of $f_A(c[t])$, and by extension one can use $\norm{\P_A(:,i)}$ and $\norm{\S_A(i,:)}$ to measure the overall influence of state $i$ in $f_A$.
Since our goal is to approximate a given WTA by a smaller WTA obtained by
removing some states in the original one, we shall proceed by removing those
states with overall less influence on the computation of $f$.
But because there are infinitely many WTA computing $f$, we need to first fix a particular representation for $f$ before we can remove the less influential states. In particular, we seek a representation where each state is decoupled as much as possible from each other state, and where there is a clear ranking of states in terms of overall influence.
It turns out all this can be achieved by a canonical form for WTA we call the
singular value tree automaton, which provides an implicit representation for the SVD of $H_f$.
We now show conditions for the existence of such canonical form, and
in the next section we develop an algorithm for computing the it
efficiently.

Suppose $f : \Ts \to \R$ is a rank $n$ rational function such that its Hankel
matrix admits a reduced singular value decomposition $\H_f = \U \D \V^\top$.
Then we have that $\P = \U \D^{1/2}$ and $\S = \D^{1/2} \V^\top$ is a rank
decomposition for $\H_f$, and by Theorem~\ref{thm:rankfact} there exists some
minimal WTA $A$ with $f_A = f$, $\P_A = \U \D^{1/2}$ and $\S_A = \D^{1/2}
\V^\top$. We call such an $A$ a \emph{singular value tree automaton} (SVTA) for $f$.
However, these are not defined for every rational function $f$, because the
fact that columns of $\U$ and $\V$ must be unitary vectors (i.e.\ $\U^\top \U =
\V^\top \V = \I$) imposes some restrictions on which infinite Hankel matrices
$\H_f$ admit an SVD --- this phenomenon is related to the distinction between compact and non-compact operators in functional analysis.
Our next theorem gives a sufficient condition for the existence of an SVD of
$\H_f$.

We say that a function $f : \Ts \to \R$ is \emph{strongly convergent} if the
series $\sum_{t \in \Ts} |t| |f(t)|$ converges.
To see the intuitive meaning of this condition, assume that $f$ is a probability
distribution over trees in $\Ts$. In this case, strong convergence is equivalent
to saying that the expected size of trees generated from the distribution $f$ is
finite.
It turns out strong convergence of $f$ is a sufficient condition to guarantee the
existence of an SVD for $\H_f$. 

\begin{theorem}\label{thm:Hsvd}
If $f : \Ts_\Sigma \to \R$ is rational and strongly convergent, then $\H_f$ admits a
singular value decomposition.
\end{theorem}
\begin{proof}
See Appendix~\ref{proof:thm:Hsvd}.
\end{proof}

Together, Theorems~\ref{thm:rankfact} and \ref{thm:Hsvd} imply that every
rational strongly convergent $f : \Ts \to \R$ can be represented by an SVTA $A$.
If $\rank(f) = n$, then $A$ has $n$ states and for every $i \in [n]$ the $i$th
state contributes to $\H_f$ by generating the $i$th left and right singular
vectors weighted by $\sqrt{\singv_i}$, where $\singv_i = \D(i,i)$ is the $i$th singular
value.
Thus, if we desire to obtain a good approximation $\hat{f}$ to $f$ with
$\hat{n}$ states, we can take the WTA $\hat{A}$ obtained by removing the last
$n - \hat{n}$ states from $A$, which corresponds to removing from $f$ the
contribution of the smallest singular values of $\H_f$. We call such $\hat{A}$
an \emph{SVTA truncation}.
Given an SVTA $A = \wta$ and $\Proj = [\I \;|\; \mat{0}] \in \R^{\hat{n} \times
n}$, the SVTA truncation to $\hat{n}$ states can be written as
\begin{equation*}
\hat{A} =
\langle \Proj \wzero, \T(\Proj^\top,\Proj^\top,\Proj^\top), \{\Proj \winf_\sigma\} \rangle
\enspace.
\end{equation*}

Theoretical guarantees on the error induced by the SVTA truncation method are
presented in Section~\ref{sec:bounds} .
\section{Computing the Singular Value WTA}\label{sec:algorithm}

Previous section shows that in order to compute an approximation to a strongly
convergent rational function $f : \Ts \to \R$ one can proceed by truncating its
SVTA. However, the only obvious way to obtain such SVTA is by computing the SVD
of the infinite matrix $\H_f$.
In this section we show that if we are given an arbitrary minimal WTA $A$
for $f$, then we can transform $A$ into the corresponding SVTA
efficiently.%
\footnote{If the WTA given to the algorithm is not minimal, a pre-processing
step can be used to minimize the input using the algorithm from
\cite{Kiefer_Marusic_Worrell_2015}.}
In other words, given a representation of $\H_f$ as a WTA, we can compute its
SVD \emph{without} the need to operate on infinite matrices.
The key observation is to reduce the computation of the SVD of $\H_f$ to the
computation of spectral properties of the Gram matrices $\gramc=\P^\top \P$ and
$\gramt = \S \S^\top$ associated with the rank factorization $\H_f = \P
\S$ induced by some minimal WTA computing $f$.
In the case of weighted automata on strings, \cite{bpp15} recently showed a polynomial time algorithm for computing the Gram matrices of a \emph{string} Hankel matrix by solving a system of linear equations. Unfortunately, extending their approach to the tree case requires obtaining a closed-form solution to a system of quadratic equations, which in general does not exist.
Thus, we shall resort to a different algorithmic technique and show that $\gramc$ and $\gramt$ can be obtained as fixed points of a certain non-linear operator. 
This yields the iterative algorithm presented in Algorithm~\ref{algo_gram} which converges exponentially fast as shown in Theorem~\ref{gram_cvg}.
The overall procedure to transform a WTA into the corresponding
SVTA is presented in Algorithm~\ref{algo:sva}.

We start with a simple linear algebra result showing exactly how to relate the eigendecompositions of $\gramc$ and $\gramt$ with the SVD of $\H_f$.

\begin{lemma}\label{lem:svdHf}
Let $f : \Ts \to \R$ be a rational function such that its Hankel matrix $\H_f$
admits an SVD. Suppose $\H_f = \P \S$ is a rank factorization. Then the
following hold:
\begin{enumerate}
\item $\gramc=\P^\top \P$ and $\gramt = \S \S^\top$ are finite symmetric
positive definite matrices with eigendecompositions $\gramc = \V_{\Cs} \D_{\Cs}
\V_{\Cs}^\top$ and $\gramt = \V_{\Ts} \D_{\Ts} \V_{\Ts}^\top$.
\item If $\M = \D_{\Cs}^{1/2} \V_{\Cs}^\top \V_{\Ts} \D_{\Ts}^{1/2}$ has SVD
$\M = \tilde{\U} \D \tilde{\V}^\top$, then $\H_f = \U \D \V^\top$ is an
SVD, where $\U = \P \V_\Cs \D_\Cs^{-1/2} \tilde{\U}$, and $\V^\top =
\tilde{\V}^\top \D_\Ts^{-1/2} \V_\Ts^\top \S$.
\end{enumerate}
\end{lemma}
\begin{proof}
The proof follows along the same lines as that of \cite[Lemma~7]{bpp15}.
\end{proof}

Putting together Lemma~\ref{lem:svdHf} and the proof of
Theorem~\ref{thm:rankfact} we see that given a minimal WTA computing a strongly
convergent rational function, Algorithm~\ref{algo:sva} below will compute the
corresponding SVTA. Note the algorithm depends on a procedure for computing the
Gram matrices $\gramt$ and $\gramc$. In the remaining of this section we
present one of our main results: a linearly convergent iterative algorithm for computing these matrices.

\begin{algorithm}
\caption{\texttt{ComputeSVTA}}\label{algo:sva}
\begin{algorithmic}
\REQUIRE A strongly convergent minimal WTA $A$
\ENSURE The corresponding SVTA
\STATE $\gramc, \gramt \leftarrow$ \texttt{GramMatrices}($A$)
\STATE Let $\gramt = \mat{V}_{\Ts}\mat{D}_{\Ts}\mat{V}_{\Ts}^\top$ and 
$\gramc = \mat{V}_{\Cs}\mat{D}_{\Cs}\mat{V}_{\Cs}^\top$ be the 
eigendecompositions of $\gramt$ and $\gramc$
\STATE Let $\mat{M} = \mat{D}_{\Cs}^{1/2} \mat{V}_{\Cs}^\top \mat{V}_{\Ts} 
\mat{D}_{\Ts}^{1/2}$ and let $\mat{M} = \mat{UDV}^\top$ 
be the singular value decomposition of $\mat{M}$
\STATE Let $\Q = \V_{\Cs} \D^{-1/2}_{\Cs} \U \D^{1/2}$
\RETURN $A^\mat{Q}$ 
\end{algorithmic}
\end{algorithm}

Let $A = \wta$ be a strongly convergent WTA of dimension $n$ computing a
function $f$. We now show how the Gram matrix $\gramt$ can be approximated
using a simple iterative scheme. 
Let $A^\kron = \langle\wzero^\kron,\T^\kron, \{\winf_\sigma^{\kron}\}\rangle$
where $\wzero^\kron = \wzero\kron\wzero$, $\T^\kron= \T\kron\T \in \R^{n^2\times
n^2 \times n^2}$ and  $\winf_\sigma^\kron = \winf_\sigma\kron\winf_\sigma$ for
all $\sigma\in\Sigma$. 
It is shown in \cite{Berstel_Reutenauer_1982} that $A^\kron$ computes the
function $f_{A^\kron}(t) = f(t)^2$.
Note we have $\gramt=\S \S^\top = \sum_{t\in \Ts}
\winf(t)\winf(t)^\top$, hence $\vec{s} \eqdef \vectorize(\gramt) = \sum_{t\in
\Ts} \winf^\kron(t)$ since $\winf^\kron(t) = \vectorize(\winf(t)
\winf(t)^\top)$.
Thus, computing the Gram matrix $\gramt$ boils down to computing the
vector $\vec{s}$. The following theorem shows that this can be done by
repeated applications of a non-linear operator until convergence to a fixed point. 

\begin{theorem}\label{thm:fpoint}
Let $F:\R^{n^2}\to\R^{n^2}$ be the mapping defined by $F(\vec{v}) =
\T^\kron(\mat{I},\vec{v},\vec{v}) + \sum_{\sigma \in \Sigma}
\winf^\kron_\sigma$. Then the following hold:
\begin{itemize}
\item[(i)] $\vec{s}$ is a fixed-point of $F$; i.e.\ $F(\vec{s}) = \vec{s}$.
\item[(ii)] $\vec{0}$ is in the basin of attraction of $\vec{s}$; i.e.\
$\lim_{k \to \infty} F^k(\vec{0}) = \vec{s}$.
\item[(iii)] The iteration defined by $\vec{s}_0=\vec{0}$ and
$\vec{s}_{k+1}=F(\vec{s}_k)$ converges linearly to $\vec{s}$; i.e. there exists
$0<\rho<1$ such that $\norm{\vec{s}_k - \vec{s}}_2 \leq \bigo{\rho^k}$.
\end{itemize} 
\end{theorem}
\begin{proof}
See Appendix~\ref{proof:thm:fpoint}.
\end{proof}

Though we could derive a similar iterative algorithm for computing $\gramc$, it turns out that knowledge of $\vec{s} = \vectorize(\gramt)$ provides an alternative, more efficient procedure for obtaining $\gramc$.
Like before, we have $\gramc = \P^\top\P = \sum_{c\in\Cs} \wzero(c)
\wzero(c)^\top$ and $\wzero^\kron(c) = \wzero(c)\kron \wzero(c)$ for all
$c\in\Cs$, hence $\vec{q} \eqdef \vectorize(\gramc) = \sum_{c\in \Cs}
\wzero^\kron(c)$.
By defining the matrix $\mat{E} = \T^\kron(\mat{I},\vec{s},\mat{I}) +\T^\kron(\mat{I},\mat{I},\vec{s})$ which only depends on $\T$ and $\vec{s}$, we can use the expression ${\wzero^\kron}^\top(c) = {\wzero^\kron}^\top \ctm_{A^\kron}(c)$
to see that:
\begin{align*}
\vec{q}^\top 
= 
\sum_{c\in \Cs} (\wzero^\kron)^\top \ctm_{A^\kron}(c) 
= 
(\wzero^\kron)^\top \sum_{k \geq 0} \mat{E}^k 
= 
(\wzero^\kron)^\top (\I - \mat{E})^{-1}\enspace,
\end{align*}
where we used the facts $\mat{E}^k = \sum_{c\in \Cs^k} \ctm_{A^\kron}(c)$ and $\rho(\mat{E}) < 1$ shown in the proof of Theorem~\ref{thm:fpoint}.

Algorithm~\ref{algo_gram} summarizes the overall approximation procedure for the
Gram matrices, which can be done to an arbitrary precision.
There, $\reshape(\cdot, n \times n)$ is an operation that takes an
$n^2$-dimensional vector and returns the $n \times n$ matrix whose first column
contains the first $n$ entries in the vector and so on.
Theoretical guarantees on the convergence rate of this algorithm are given in
the following theorem.

\begin{theorem}\label{gram_cvg}
There exists $0<\rho<1$ such that after $k$ iterations in
Algorithm~\ref{algo_gram}, the approximations  $\gramcapprox$ and $\gramtapprox$ 
satisfy $\norm{\gramc- \gramcapprox}_F \leq \mathcal{O}(\rho^k)$ and  
$\norm{\gramt- \gramtapprox}_F \leq \mathcal{O}(\rho^k)$.
\end{theorem}
\begin{proof}
See Appendix~\ref{proof:gram_cvg}.
\end{proof}

\begin{algorithm}
  \caption{\texttt{GramMatrices}}
  \label{algo_gram}
\begin{algorithmic}
\REQUIRE A strongly convergent minimal WTA $A = \wta$
\ENSURE Gram matrices 
$\gramcapprox \simeq \sum_{c\in \CF} \wzero_A(c)\wzero_A(c)^\top$ 
and $\gramtapprox \simeq \sum_{t\in \TF} \winf_A(t) \winf_A(t)^\top$
\STATE Let $\T^\kron= \T\kron\T \in \R^{n^2\times n^2 \times n^2}$, and let
$\winf_\sigma^\kron = \winf_\sigma\kron\winf_\sigma \in \R^{n^2}$ 
for all $\sigma\in\Sigma$. 
\STATE Let $\mat{I}$ be the $n^2\times n^2$ identity matrix and let $\vec{s} =
\vec{0} \in \R^{n^2}$ 
\REPEAT
\STATE $\vec{s} \leftarrow \T^\kron(\mat{I}, \vec{s}, \vec{s}) + 
\sum_{\sigma\in\Sigma} \winf_\sigma^\kron$
\UNTIL convergence
\STATE $\vec{q} = (\wzero \kron \wzero)^\top \left(\mat{I} -
\T^\kron(\mat{I},\mat{I},\vec{s}) - \T^\kron(\mat{I},\vec{s},\mat{I})
\right)^{-1} $ 
\STATE $\gramtapprox = \reshape(\vec{s}, n\times n)$
\STATE $\gramcapprox = \reshape(\vec{q}, n\times n)$
\RETURN $\gramcapprox, \gramtapprox$
\end{algorithmic}
\end{algorithm}

\section{Approximation Error of an SVTA Truncation}\label{sec:bounds}
In this section, we analyze the approximation error induced by the truncation 
of an SVTA. We recall that given a SVTA $A= \wta$, its truncation to $\hat{n}$ 
states is the automaton
\begin{equation*}
\hat{A} =
\langle \Proj \wzero, \T(\Proj^\top,\Proj^\top,\Proj^\top), \{\Proj \winf_\sigma\} \rangle
\end{equation*}
where $\Proj = [\I \;|\; \mat{0}] \in \R^{\hat{n} \times n}$ is the projection
matrix which removes the states associated with the $n-\hat{n}$ smallest
singular values of the Hankel matrix.

Intuitively, the states associated with the smaller singular values
are the ones with the less influence on the Hankel matrix, thus they should
also be the states having the less effect on the computation of the SVTA. 
The following theorem support this intuition by showing a fundamental relation between the singular
values of the Hankel matrix of a rational function $f$ and the parameters of
the SVTA computing it.

\begin{theorem}\label{prop:singv_and_SVTA_params}
Let $A = \wtaSigma$ be a SVTA with $n$ states realizing a function $f$ and let 
$\singv_1 \geq \singv_2 \geq \cdots \geq \singv_n$ be the singular
values of the Hankel matrix $\mat{H}_f$.
Then, for any $t\in \Ts$ and $i,j,k\in [n]$ the following hold:
 \begin{itemize}
 \item $|\winf(t)_i| \leq \sqrt{\singv_i}$ ,
 \item $|\wzero_i| \leq \sqrt{\singv_i}$ , and
 \item $|\T(i,j,k)| \leq \min\{\frac{\sqrt{\singv_i}}{\sqrt{\singv_j}\sqrt{\singv_k}},
\frac{\sqrt{\singv_j}}{\sqrt{\singv_i}\sqrt{\singv_k}},\frac{\sqrt{\singv_k}}{\sqrt{\singv_i}\sqrt{\singv_j}}\}$.
 \end{itemize}
\end{theorem}
\begin{proof}
See Appendix~\ref{proof:prop:singv_and_SVTA_params}.
\end{proof}

Two important properties of SVTAs follow from this proposition. First, the 
fact that $|\winf(t)_i| \leq \sqrt{\singv_i}$ implies that the weights associated with 
states corresponding to small singular values are small. Second, this proposition 
gives us some intuition on how the states of an SVTA interact with each other.
To see this, let $\M = \T(\wzero, \I, \I)$ and remark that for a tree $t=(t_1,t_2) \in \Ts$ we
have $f(t) = \winf(t_1)^\top \M \winf(t_2)$. Using the previous theorem one can show
that 
$$|\M(i,j)| \leq n\ \sqrt{\frac{\min\{\singv_i,\singv_j\}}{\max\{\singv_i,\singv_j\}}} \enspace,$$
which tells us that two states corresponding to singular values
far away from each other have very little interaction in the computations of the automata.

Theorem~\ref{prop:singv_and_SVTA_params} is key to proving the following theorem, 
which is the main result of this section. It shows how the approximation
error induced by the truncation of an SVTA is impacted by the magnitudes of the singular
values associated with the removed states.
 
\begin{theorem}\label{thm:svtabound}
Let $A = \wtaSigma$ be a SVTA with $n$ states realizing a function $f$ and let 
$\singv_1 \geq \singv_2 \geq \cdots \geq \singv_n$ be the singular
values of the Hankel matrix $\mat{H}_f$. Let $\hat{f}$ be the function computed
by the SVTA truncation of $A$ to $\hat{n}$ states. The following holds for any $\varepsilon > 0$:
\begin{itemize}
\item  For any tree $t\in \Ts$ of size $M$,
 if ${\displaystyle M < \frac{\log\left(\frac{1}{\singv_{\hat{n}+1}}\right) + \log\left(\varepsilon\right) }{2\log n} }$
then $|f(t) - \hat{f}(t)| < \varepsilon$.
\item Furthermore, if ${\displaystyle M < \frac{\log\left(\frac{1}{\singv_{\hat{n}+1}}\right) + \log(\varepsilon)}{\log(4 |\Sigma| n^2)} - 1}$ then
$\displaystyle\sum_{t:\size{t}<M} |f(t)-\hat{f}(t)|< \varepsilon$.
\end{itemize}
\end{theorem}
\begin{proof}
See Appendix~\ref{proof:thm:svtabound}.
\end{proof}
 
Since $\singv_{\hat{n}+1} > \singv_{\hat{n}+2} > \cdots > \singv_n$, 
this theorem shows that the smaller the singular values associated with the 
removed states are, the better will be the approximation. As a direct 
consequence, the error introduced by the truncation grows with the 
number of states removed.
The dependence on the size of the trees comes from the propagation 
of the error during the contractions of the tensor $\hat{\T}$ of the 
truncated SVTA.

The decay of singular values can be very slow in the worst case, but in practice is not unusual to observe an exponential decay on the tail. For example, this is shown to be the case for the SVTA we compute in Section~\ref{sec:experiments}. Assuming such an exponential decay of the form $\singv_i = C \theta^i$ for some $0 < \theta < 1$, the second bound above on the size of the trees for which $\sum_{\size{t} < M} |f(t) - \hat{f}(t)| < \varepsilon$ specializes to
\begin{equation*}
\frac{(\hat{n}+1)\log(1/\theta)+\log(\varepsilon)+\log(C)}{\log(4|\Sigma|n^2)}\enspace.
\end{equation*}
It is interesting to observe that the dependence of this bound on the number of total/removed states is $O(\hat{n} / \log(n))$.

\section{Experiments}\label{sec:experiments}

In this section, we assess the performance of our method on a model arising from
real-world data, by using a PCFG learned from a text corpus as our initial
model. Before presenting our experimental setup and results, we recall the
standard mapping between WCFG and WTA.

\subsection{Converting WCFG to WTA}

A \emph{weighted context-free grammar} (WCFG) in Chomsky normal form is a tuple 
$G = \langle \mathcal{N}, \Sigma, \mathcal{R}, \weight \rangle$ 
where $\mathcal{N}$ is the finite set of nonterminal symbols, $\Sigma$ is the finite 
set of words, with $\Sigma \cap \mathcal{N} = \emptyset$, $\mathcal{R}$ is a set
of rules having the form $(a \to bc) $, $(a \to x)$ or $(\to a)$ for
$a,b,c\in\mathcal{N}, x\in\Sigma$, and $ \weight: \mathcal{R} \to \R$ is the
weight function which is extended to the set of all possible rules by letting $
\weight(\delta) = 0$ for all rules $\delta \not\in \mathcal{R}$. 

A WCFG $G$ assigns a weight to each derivation tree $\tau$ of the grammar given
by $\weight(\tau) = \prod_{\delta \in \mathcal{R}}
w(\delta)^{\sharp_\delta(\tau)}$ (where $\sharp_\delta(\tau)$ is the number of
times the rule $\delta$ appears in $\tau$), and it computes a function
$f_G: \Sigma^+ \to \R$ defined by $f_G(w) = \sum_{\tau \in T(w)} \weight(\tau)$
for any $w\in \Sigma^+$, where $T(w)$ is the set of trees deriving the word 
$w$.

Given a WCFG $G$, we can build a WTA that assigns to each binary tree $
t\in\Ts_\Sigma$ the sum of the weights of all derivation trees of $G$ having the
same topology as $t$.
Let $G = \langle \mathcal{N}, \Sigma, \mathcal{R}, w\rangle$ be a WCFG in
normal form with $\mathcal{N} = [n]$. Let $A = \wtaSigma$ be the WTA with $n$
states defined by $\wzero(i) = \weight(\to i)$ for all $i\in [n]$, $\T(i,j,k) =
\weight(i\to jk)$ for all $i,j,k\in[n]$, and $\winf_\sigma(i) = \weight(i\to
\sigma)$ for all $i\in[n]$, $\sigma \in \Sigma$. 
Then for all $w \in \Sigma^+$ we have $f_G(w) = \sum_{t\in\Ts_\Sigma:
\yield{t}=w} f_A(t)\enspace.$ 
It is important to note that in this conversion the number of states in $A$
corresponds to the number of non-terminals in $G$.
A similar construction can be used to convert any WTA  to a WCFG where
each state in the WTA is mapped to a non-terminal in the WCFG.

\subsection{Experimental Setup and Results}

In our experiments, we used the annotated corpus of german newspaper texts
NEGRA \cite{skut97}. 
We use a standard setup, in which the first 18,602 sentences are used as a
training set, the next 1,000 sentences as a development set and the last 1,000
sentences as a test set $S_{\mathrm{test}}$.
All trees are binarized as described in \cite{cohen-13b}. We extract a binary
grammar in Chomsky normal form from the data, and then estimate its
probabilities using maximum likelihood. The resulting PCFG has $n = 211$
nonterminals.
We compare our method against the ones described in \cite{cohen-13a}, who use
tensor decomposition algorithms \cite{chi-11} to decompose the tensors of an
underlying PCFG.%
\footnote{We use two tensor decomposition algorithms from the tensor Matlab
toolbox: \texttt{pqnr}, which makes use of projected quasi-Newton and
\texttt{mu}, which uses a multiplicative update. See
\url{http://www.sandia.gov/~tgkolda/TensorToolbox/index-2.6.html}.}

We used three evaluation measures: $\ell_2$ distance (between the functions of
type $\Ts_\Sigma \to \R$ computed by the original WTA and the one computed by
its approximation), perplexity on a test set, and parsing accuracy on a test set
(comparing the tree topology of parses using the bracketing F-measure).
Because the number of states on a WTA and the CP-rank of tensor decomposition
method are not directly comparable, we plotted the results using the number of
parameters needed to specify the model on the horizontal axis.
This number is equal to $\hat{n}^3$ for a WTA with $\hat{n}$ states, and it
is equal to $3 R n$ when the tensor $\T$ is approximated with a tensor of
CP-rank $R$ (note in both cases these are the number of parameters needed to
specify the tensor occurring in the model).

The $\ell_2$ distance between the original function $f$ and its minimization
$\hat{f}$, $\norm{f-\hat{f}}_2^2 = \sum_{t \in \Ts} (f(t)-\hat{f}(t))^2$, can
be approximated to an arbitrary precision using the Gram matrices of the
corresponding WTA (which follows from observing that $(f-\hat{f})^2$ is
rational).
The perplexity of $\hat{f}$ is defined by $2^{-H_\mathrm{test}}$, where
$H_\mathrm{test} = \sum_{t\in S_{\mathrm{test}}} f(t) \log_2 \hat{f}(t)$ and
both $f$ and $\hat{f}$ have been normalized to sum to one over the test set.
The results are plotted in Figure~\ref{fig:l2_perplex}, where an horizontal
dotted line represents the performance of the original model. We see that
our method outperforms the tensor decomposition methods both in terms
of $\ell_2$ distance and perplexity. We also remark that our method obtains very smooth curves, which
comes from the fact that it does not suffer from local optima problems like the
tensor decomposition methods.

\begin{figure}
\begin{center}

\includegraphics[scale=0.4]{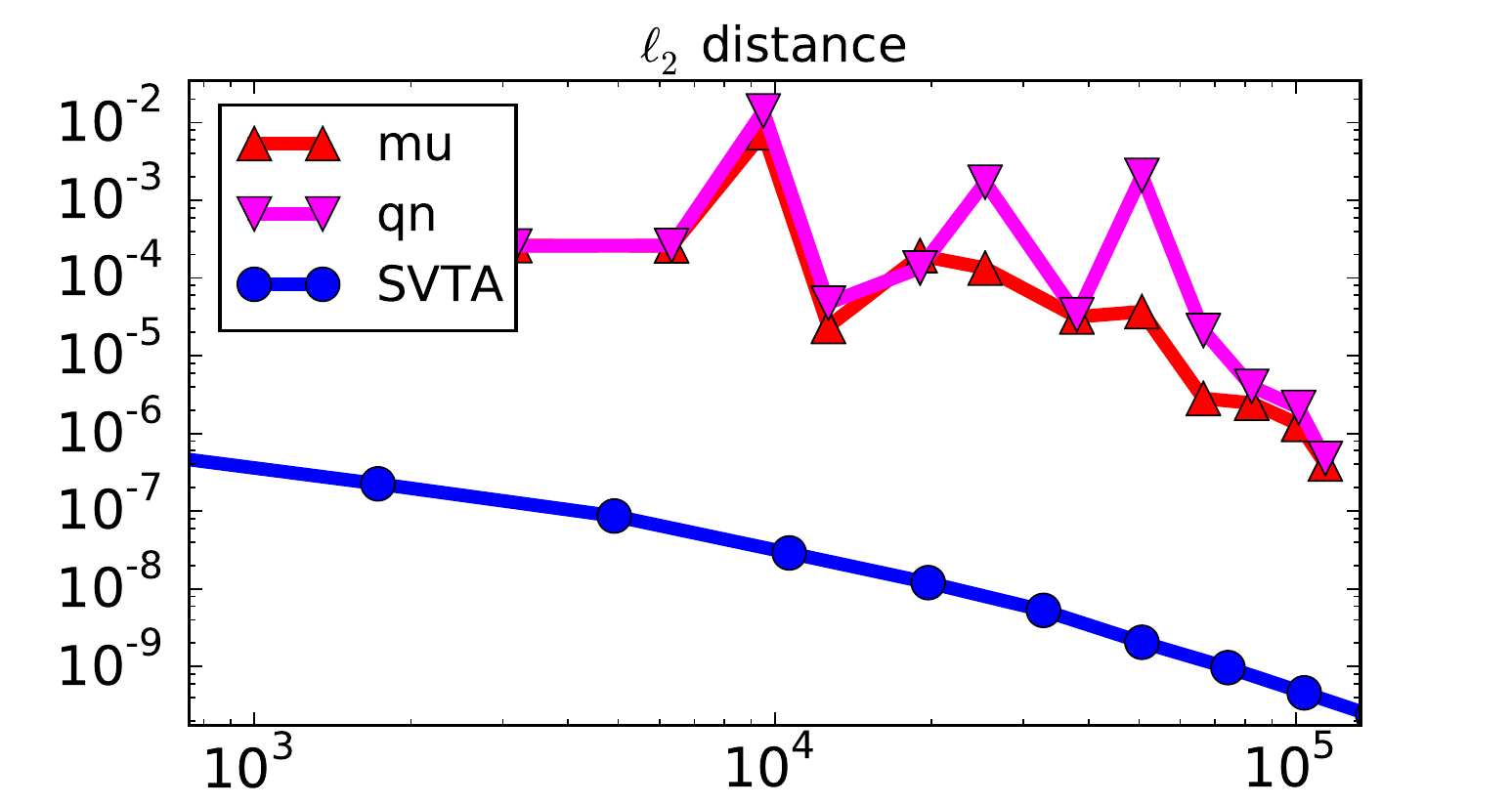}\hspace{-0.5cm}%
\includegraphics[scale=0.4]{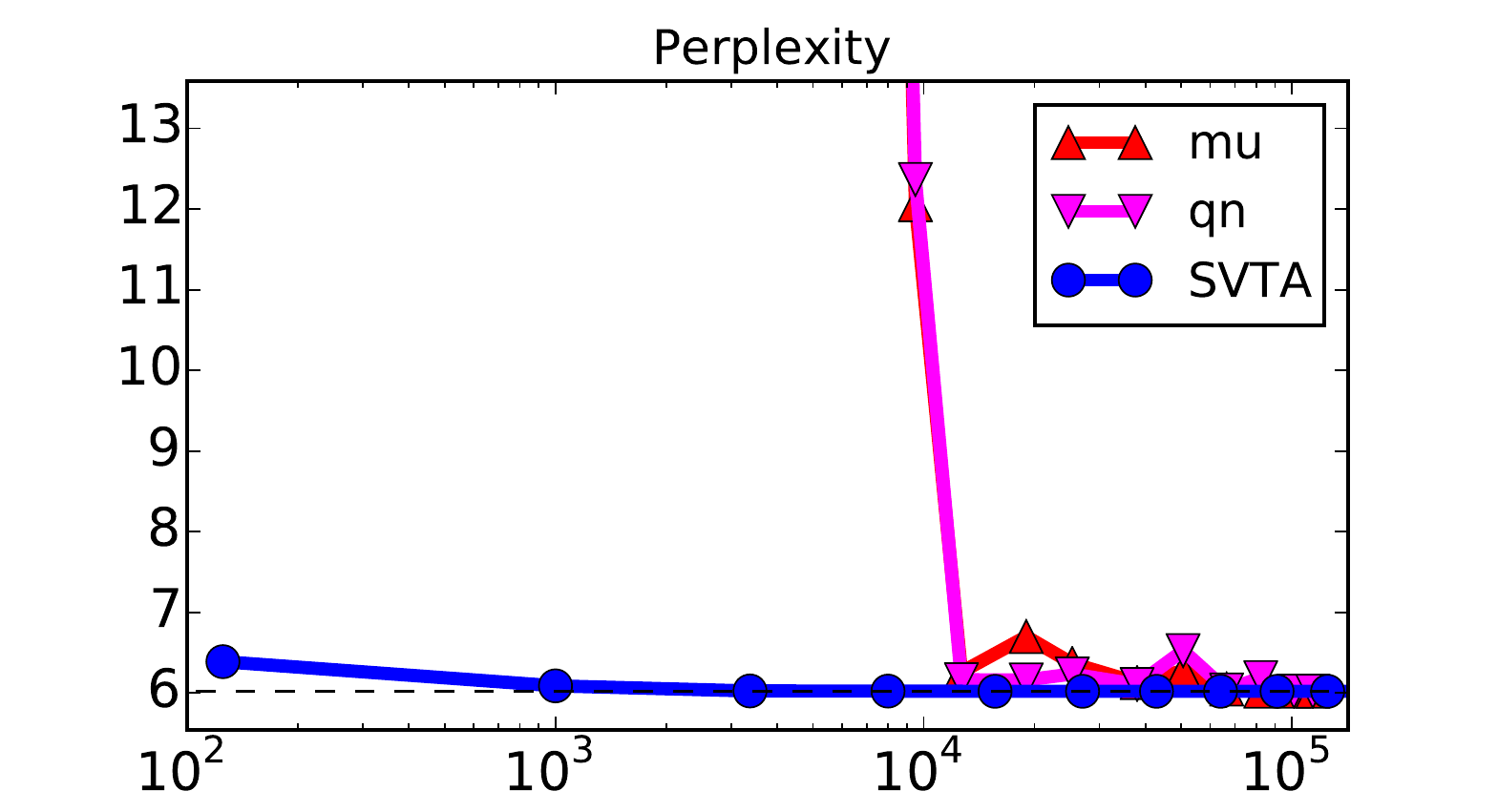}
\vspace*{-1em}
\caption{(top) $\ell_2$ distance between functions. (bottom) Perplexity on the
test set. The $x$ axis denotes in both cases the number of parameters used by the approximation.}
\label{fig:l2_perplex}
\end{center}
\end{figure}

\begin{figure}
\begin{center}
\includegraphics[scale=0.4]{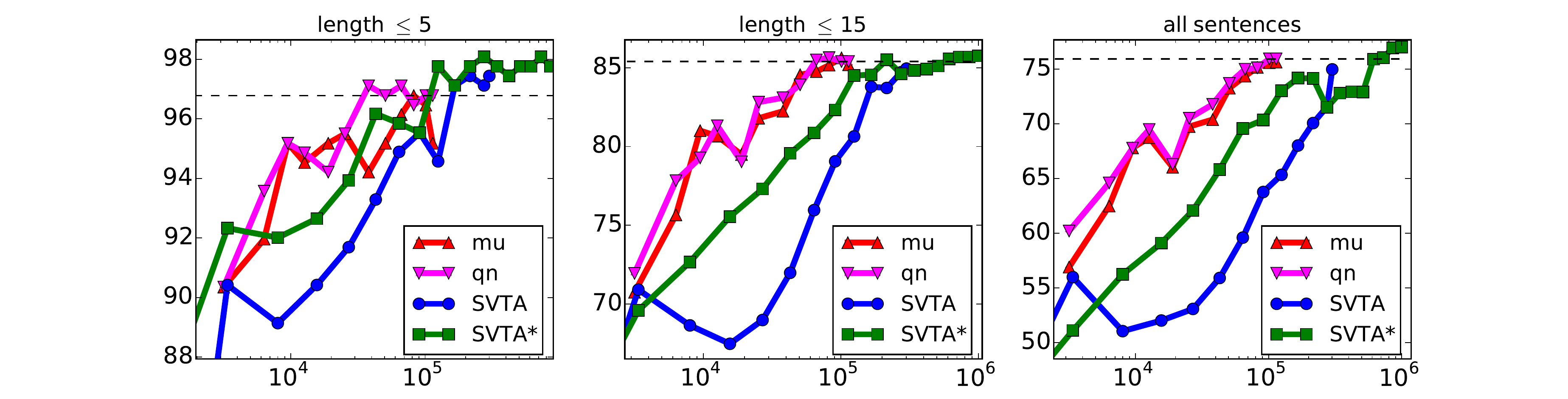}

\caption{Parsing accuracy on the test set for different sentence lengths. The $x$-axis
denote the number of parameters used by the approximation. The $y$-axis denotes bracketing
accuracy.}
\label{fig:parsing}
\end{center}
\end{figure}

For parsing we use minimum Bayes risk decoding, maximizing the sum of the
marginals for the nonterminals in the grammar, essentially choosing the best tree
topology given a string \cite{goodman-96}.
The results for various length of sentences are shown in
Figure~\ref{fig:parsing}, where we see that our method does not perform as well
as the tensor decomposition methods in terms of parsing accuracy on long
sentences. 
In this figure, we also plotted the results for a slight modification of our
method (SVTA$^*$) that is able to achieve competitive performances.  
The SVTA$^*$ method gives more importance to long sentences in the minimization
process. This is done by finding the highest constant $\gamma > 0$ such that the
function $f_\gamma : t \mapsto \gamma^{\size{t}} f(t)$ is still strongly
convergent.
This function is then approximated by a low-rank WTA computing $\hat{f}_\gamma$,
and we let $\hat{f} :t \mapsto \gamma^{-\size{t}} \hat{f}_\gamma(t)$ (which is
rational).  In our experiment, we used $\gamma=2.4$.
While the SVTA$^*$ method improved the parsing accuracy, it had no significant
repercussion on the $\ell_2$ and perplexity measures.
We believe that the parsing accuracy of our method could be further improved.
Seeking techniques that combines the benefits of SVTA and previous works is
a promising direction.

\section{Conclusion}\label{sec:conclusion}

We described a technique for approximate minimization of WTA, yielding a model
smaller than the original one which retains good approximation properties. Our
main algorithm relies on a singular value decomposition of an infinite Hankel
matrix induced by the WTA. We provided theoretical guarantees on the error
induced by our minimization  method.
Our experiments with real-world parsing data show
that the minimized WTA, depending on the number of singular values used,
approximates well the original WTA on three measures: perplexity, bracketing
accuracy and $\ell_2$ distance of the tree weights.
Our work has connections with spectral learning techniques for WTA, and exhibits
similar properties as those algorithms; e.g.\ absence of local optima.
In future work we plan to investigate the applications of our approach to the
design and analysis of improved spectral learning algorithms for WTA.

\bibliographystyle{apalike}
\bibliography{main}

\begin{thebibliography}{}

\bibitem[Bailly et~al., 2009]{denis}
Bailly, R., Denis, F., and Ralaivola, L. (2009).
\newblock Grammatical inference as a principal component analysis problem.
\newblock In {\em Proceedings of ICML}.

\bibitem[Bailly et~al., 2010]{bailly2010spectral}
Bailly, R., Habrard, A., and Denis, F. (2010).
\newblock A spectral approach for probabilistic grammatical inference on trees.
\newblock In {\em Proceedings of ALT}.

\bibitem[Balle et~al., 2014]{mlj13spectral}
Balle, B., Carreras, X., Luque, F., and Quattoni, A. (2014).
\newblock Spectral learning of weighted automata: A forward-backward
  perspective.
\newblock {\em Machine Learning}.

\bibitem[Balle et~al., 2015]{bpp15}
Balle, B., Panangaden, P., and Precup, D. (2015).
\newblock A canonical form for weighted automata and applications to
  approximate minimization.
\newblock In {\em Proceedings of LICS}.

\bibitem[Berstel and Reutenauer, 1982]{Berstel_Reutenauer_1982}
Berstel, J. and Reutenauer, C. (1982).
\newblock Recognizable formal power series on trees.
\newblock {\em Theoretical Computer Science}.

\bibitem[Boots et~al., 2011]{psr}
Boots, B., Siddiqi, S., and Gordon, G. (2011).
\newblock Closing the learning planning loop with predictive state
  representations.
\newblock {\em International Journal of Robotics Research}.

\bibitem[Bozapalidis and Louscou-Bozapalidou,
  1983]{Bozapalidis_Louscou-Bozapalidou_1983}
Bozapalidis, S. and Louscou-Bozapalidou, O. (1983).
\newblock The rank of a formal tree power series.
\newblock {\em Theoretical Computer Science}.

\bibitem[Chi and Kolda, 2012]{chi-11}
Chi, E.~C. and Kolda, T.~G. (2012).
\newblock On tensors, sparsity, and nonnegative factorizations.
\newblock {\em SIAM Journal on Matrix Analysis and Applications}.

\bibitem[Cohen and Collins, 2012]{cohen-12c}
Cohen, S.~B. and Collins, M. (2012).
\newblock Tensor decomposition for fast parsing with latent-variable {PCFGs}.
\newblock In {\em Proceedings of NIPS}.

\bibitem[Cohen et~al., 2013a]{cohen-13a}
Cohen, S.~B., Satta, G., and Collins, M. (2013a).
\newblock Approximate {PCFG} parsing using tensor decomposition.
\newblock In {\em Proceedings of {NAACL}}.

\bibitem[Cohen et~al., 2013b]{cohen-13b}
Cohen, S.~B., Stratos, K., Collins, M., Foster, D.~P., and Ungar, L. (2013b).
\newblock Experiments with spectral learning of latent-variable {PCFGs}.
\newblock In {\em Proceedings of {NAACL}}.

\bibitem[Cohen et~al., 2014]{cohen-14b}
Cohen, S.~B., Stratos, K., Collins, M., Foster, D.~P., and Ungar, L. (2014).
\newblock Spectral learning of latent-variable {PCFGs}: Algorithms and sample
  complexity.
\newblock {\em Journal of Machine Learning Research}.

\bibitem[Conway, 1990]{conway1990course}
Conway, J.~B. (1990).
\newblock {\em A course in functional analysis}.
\newblock Springer.

\bibitem[El~Ghaoui, 2002]{invperturbation}
El~Ghaoui, L. (2002).
\newblock Inversion error, condition number, and approximate inverses of
  uncertain matrices.
\newblock {\em Linear algebra and its applications}.

\bibitem[Goodman, 1996]{goodman-96}
Goodman, J. (1996).
\newblock Parsing algorithms and metrics.
\newblock In {\em {Proceedings of ACL}}.

\bibitem[Hsu et~al., 2012]{hsu2012spectral}
Hsu, D., Kakade, S.~M., and Zhang, T. (2012).
\newblock A spectral algorithm for learning hidden {Markov} models.
\newblock {\em Journal of Computer and System Sciences}.

\bibitem[Kiefer et~al., 2015]{Kiefer_Marusic_Worrell_2015}
Kiefer, S., Marusic, I., and Worrell, J. (2015).
\newblock {\em Minimisation of Multiplicity Tree Automata}.

\bibitem[Kulesza et~al., 2015]{kulesza2015low}
Kulesza, A., Jiang, N., and Singh, S. (2015).
\newblock Low-rank spectral learning with weighted loss functions.
\newblock In {\em Proceedings of AISTATS}.

\bibitem[Kulesza et~al., 2014]{kulesza2014low}
Kulesza, A., Rao, N.~R., and Singh, S. (2014).
\newblock {Low-Rank Spectral Learning}.
\newblock In {\em Proceedings of AISTATS}.

\bibitem[Ortega, 1990]{Ortega_1990}
Ortega, J.~M. (1990).
\newblock {\em Numerical analysis: a second course}.
\newblock Siam.

\bibitem[Skut et~al., 1997]{skut97}
Skut, W., Krenn, B., Brants, T., and Uszkoreit, H. (1997).
\newblock {An annotation scheme for free word order languages}.
\newblock In {\em Conference on Applied Natural Language Processing}.

\end{thebibliography}

\appendix
\section{Proof of Theorem~\ref{thm:rankfact}}
\label{proof:thm:rankfact}
\begin{theorem*}%
Let $f : \Ts \to \R$ be rational. If $\H_f = \P \S$ is a rank factorization,
then there exists a minimal WTA $A$ computing $f$ such that $\P_A = \P$ and
$\S_A = \S$.
\end{theorem*}
\begin{proof}
Let $n = \rank(f)$. Let $B$ be an arbitrary minimal WTA computing $f$.
Suppose $B$ induces the rank factorization $\H_f = \P' \S'$.
Since the columns of both $\P$ and $\P'$ are basis for the column-span of
$\H_f$, there must exists a change of basis $\Q \in \R^{n \times n}$ between
$\P$ and $\P'$. That is, $\Q$ is an invertible matrix such that $\P' \Q = \P$.
Furthermore, since $\P' \S' = \H_f = \P \S = \P' \Q \S$ and $\P'$ has full
column rank, we must have $\S' = \Q \S$, or equivalently, $\Q^{-1} \S' = \S$.
Thus, we let $A = B^{\Q}$, which immediately verifies $f_A = f_B = f$.
It remains to be shown that $A$ induces the rank factorization $\H_f = \P \S$.
Note that when proving the equivalence $f_A = f_B$ we already showed $\winf_A(t)
= \Q^{-1} \winf_B(t)$, which means we have $\S_A = \Q^{-1} \S' = \S$.
To show $\P_A = \P' \Q$ we need to show that for any $c \in \Cs$ we have
$\wzero_{A}(c)^\top = \wzero_B(c)^\top \Q$. This will immediately follow if we
show that $\ctm_A(c) = \Q^{-1} \ctm_B(c) \Q$. If we proceed by induction on
$\drop{c}$, we see the case $c = \gap$ is immediate, and for $c = (c',t)$ we get
\begin{align*}
\ctm_A((c',t)) &= (\T(\Q^{-\top},\Q,\Q))(\I,\ctm_A(c'),\winf_A(t)) \\
&=
(\T(\Q^{-\top},\Q,\Q))(\I,\Q^{-1} \ctm_B(c') \Q, \Q^{-1} \winf_B(t)) \\
&=
\T(\Q^{-\top},\ctm_B(c') \Q,\winf_B(t)) \\
&=
\Q^{-1} \T(\I,\ctm_B(c'),\winf_B(t)) \Q \enspace.
\end{align*}
Applying the same argument mutatis mutandis for $c = (t,c')$ completes the
proof.
\end{proof}

\section{Proof of Theorem~\ref{thm:Hsvd}}
\label{proof:thm:Hsvd}
\begin{theorem*}
If $f : \Ts_\Sigma \to \R$ is rational and strongly convergent, then $\H_f$ admits a
singular value decomposition.
\end{theorem*}
\begin{proof}
The result will follow if we show that $\H_f$ is the matrix of a compact
operator on a Hilbert space \cite{conway1990course}. The main obstruction to
this approach is that the rows and columns of $\H_f$ are indexed by different
objects (trees vs.\ contexts). Thus, we will need to see $\H_f$ as an operator
on a larger space that contains both these objects.

Recall we have $\Ts_\Sigma \subset \Ts_{\Sigma'}$ and $\Cs_{\Sigma} \subset
\Ts_{\Sigma'}$.
Given two functions $g, g' : \Ts_{\Sigma'} \to \R$ we define their inner product
to be
$\left< g, g' \right> = \sum_{t' \in \Ts_{\Sigma'}} g(t') g'(t')$.
Let $\norm{g} = \sqrt{|\left< g, g \right>|}$ be the induced norm and let $\sT$
be the space of all functions $g : \Ts_{\Sigma'} \to \R$ such that $\norm{g} <
\infty$.
Note that $\sT$ with a Hilbert space, and that since $\Ts_{\Sigma'}$ is
countable, it actually is a separable Hilbert space isomorphic to $\ell^2$, the
spaces of infinite square summable sequences.
Given set $\X \subset \Ts_{\Sigma'}$ we define $\sT(\X) = \{ g \in \sT \;|\;
g(t') = 0, t' \in \Ts_{\Sigma'} \setminus \X \}$.

Now let $C_f : \sT \to \sT$ be the linear operator on $\sT$ given by
\begin{equation*}
(C_f g)(t') =
\begin{cases}
\sum_{t \in \Ts_\Sigma} f(t'[t]) g(t) & \text{if $t' \in \Cs_\Sigma$} \\
0 & \text{if $t' \notin \Cs_\Sigma$} \enspace.
\end{cases}
\end{equation*}
Now note that by construction we have $\sT(\Ts_\Sigma) \subseteq \Ker(C_f)$ and
$\Im(C_f) \subseteq \sT(\Cs_\Sigma)$.
Hence, a simple calculation shows that given the decompositions $C_f :
\sT(\Ts_\Sigma)^\bot \oplus \sT(\Ts_\Sigma) \to \sT(\Cs_\Sigma) \oplus
\sT(\Cs_\Sigma)^\bot$, the matrix of $C_f$ is
\begin{equation*}
\mat{C}_f =
\left[
\begin{array}{cc}
\H_f & \mat{0} \\
\mat{0} & \mat{0}
\end{array}
\right] \enspace.
\end{equation*}
Thus, if $C_f$ is a compact operator, then $\H_f$ admits an SVD.
Since $\H_f$ has finite rank, we only need to show that $C_f$ is a bounded
operator.

Given $c \in \Cs_\Sigma$ we define $f_c \in \sT(\Ts_\Sigma)$ given by $f_c(t) =
f(c[t])$ for $t \in \Ts_\Sigma$.
Now let $g \in \sT$ with $\norm{g} = 1$ and recall $C_f$ is bounded if
$\norm{C_f g} < \infty$ for every $g \in \sT$ with $\norm{g} = 1$.
Indeed, because $f$ is strongly convergent we have:
\begin{align*}
\norm{C_f g}^2
&=
\sum_{t' \in \Ts_{\Sigma'}} (C_f g)(t')^2\\
&=
\sum_{c \in \Cs_{\Sigma}} (C_f g)(c)^2\\
&=
\sum_{c \in \Cs_{\Sigma}} \left(\sum_{t \in \Ts_\Sigma} f(c[t]) g(t)\right)^2\\
&=
\sum_{c \in \Cs_{\Sigma}} \left< f_c, g \right>^2 \\
&\leq
\norm{g}^2 \sum_{c \in \Cs_{\Sigma}} \norm{f_c}^2\\
&=
\sum_{c \in \Cs_{\Sigma}} \sum_{t' \in \Ts_{\Sigma'}} f_c(t')^2\\
&=
\sum_{c \in \Cs_{\Sigma}} \sum_{t \in \Ts_{\Sigma}} f(c[t])^2\\
&=
\sum_{t \in \Ts_{\Sigma}} |t| f(t)^2 \\
&\leq
\sup_{t \in \Ts_{\Sigma}} |f(t)| \cdot \sum_{t \in \Ts_{\Sigma}} |t| |f(t)| <
\infty \enspace,
\end{align*}
where we used the Cauchy--Schwarz inequality, and the fact that $\sup_{t \in
\Ts_{\Sigma}} |f(t)|$ is bounded when $f$ is strongly convergent.
\end{proof}

\section{Proof of Theorem~\ref{thm:fpoint}}
\label{proof:thm:fpoint}
\begin{theorem*}%
Let $F:\R^{n^2}\to\R^{n^2}$ be the mapping defined by $F(\vec{v}) =
\T^\kron(\mat{I},\vec{v},\vec{v}) + \sum_{\sigma \in \Sigma}
\winf^\kron_\sigma$. Then the following hold:
\begin{itemize}
\item[(i)] $\vec{s}$ is a fixed-point of $F$; i.e.\ $F(\vec{s}) = \vec{s}$.
\item[(ii)] $\vec{0}$ is in the basin of attraction of $\vec{s}$; i.e.\
$\lim_{k \to \infty} F^k(\vec{0}) = \vec{s}$.
\item[(iii)] The iteration defined by $\vec{s}_0=\vec{0}$ and
$\vec{s}_{k+1}=F(\vec{s}_k)$ converges linearly to $\vec{s}$; i.e. there exists
$0<\rho<1$ such that $\norm{\vec{s}_k - \vec{s}}_2 \leq \bigo{\rho^k}$.
\end{itemize} 
\end{theorem*}
\begin{proof}
(i)
We have $\T^\kron(\mat{I}, \vec{s},\vec{s}) = \sum_{t,t'\in \Ts}
\T^\kron(\mat{I}, \winf^\kron(t), \winf^\kron(t')) = \sum_{t,t'\in\Ts}
\winf^\kron((t,t')) = \sum_{t \in \Ts^{\geq 1}} \winf^\kron(t)$ where $\Ts^{\geq
1}$ is the set of trees of depth at least one.
Hence $F(\vec{s}) = \sum_{t\in \Ts^{\geq 1}} \winf^\kron(t) +
\sum_{\sigma\in\Sigma} \winf^\kron_\sigma = \vec{s}$.

(ii) Let $\Ts^{\leq k}$ denote the set of all trees with depth at most $k$. We
prove by induction on $k$ that $F^k(\vec{0}) = \sum_{t\in \Ts^{\leq k}}
\winf^\kron(t)$, which implies that $\lim_{k\to \infty} F^k(\vec{0}) = \vec{s}$.
This is straightforward for $k=0$. Assuming it is true for all naturals up to
$k-1$, we have
\begin{align*}
F^k(\vec{0})
&=
\T^\kron(\mat{I}, F^{k-1}(\vec{0}), F^{k-1}(\vec{0})) + \sum_{\sigma\in\Sigma}
\winf^\kron_\sigma \\ 
&=
\sum_{t,t'\in \Ts^{\leq k-1}} \T^\kron(\mat{I}, \winf^\kron(t), \winf^\kron(t'))
+ \sum_{\sigma\in\Sigma} \winf^\kron_\sigma \\
&=
\sum_{t,t'\in \Ts^{\leq k-1}} \winf^\kron((t,t')) + \sum_{\sigma\in\Sigma}
\winf^\kron_\sigma\\
&=
\sum_{t\in \Ts^{\leq k}} \winf^\kron(t) \enspace.
\end{align*}

(iii) Let $\vec{E}$ be the Jacobian of $F$ around $\vec{s}$, we show that the
spectral radius $\rho(\mat{E})$ of $\mat{E}$ is less than one, which implies the
result by Ostrowski's theorem (see \cite[Theorem~8.1.7]{Ortega_1990}). 

Since $A$ is minimal, there exists trees $t_1,\cdots,t_n\in\Ts$ and contexts
$c_1,\cdots,c_n\in \Cs$ such that both $\{\winf(t_i)\}_{i \in [n]}$ and
$\{\wzero(c_i)\}_{i \in [n]}$ are sets of linear independent vectors in $\R^n$
\cite{bailly2010spectral}. Therefore, the sets
$\{\winf(t_i) \kron \winf(t_j)\}_{i,j \in [n]}$ and $\{\wzero(c_i) \kron
\wzero(c_j)\}_{i,j \in [n]}$ are sets of linear independent vectors in $\R^{n^2}$.
Let $\vec{v} \in \R^{n^2}$ be an eigenvector of $\mat{E}$ with eigenvalue
$\lambda \neq 0$, and let $\vec{v} = \sum_{i,j \in [n]} \beta_{i,j}
(\winf(t_i)\kron\winf(t_j))$ be its expression in terms of the basis given by
$\{\winf(t_i) \kron \winf(t_j)\}$.
For any vector $\vec{u} \in \{\wzero(c_i) \kron \wzero(c_j)\}$ we have
\begin{align*}
 \lim_{k \to \infty} \vec{u}^\top \mat{E}^k \vec{v} 
 \leq
\lim_{k \to \infty} |\vec{u}^\top \mat{E}^k \vec{v}| 
\leq
\sum_{i,j \in [n]} |\beta_{i,j}| \lim_{k \to \infty} |\vec{u}^\top
\mat{E}^k (\winf(t_i)\kron\winf(t_j))| = 0 \enspace,
\end{align*}
where we used Lemma~\ref{lemma:cctzero} in the last step.
Since this is true for any vector $\vec{u}$ in the basis $\{\wzero(c_i) \kron
\wzero(c_j)\}$, we have $\lim_{k\to\infty} \mat{E}^k\vec{v} = \lim_{k\to\infty}
|\lambda|^k \vec{v} = \vec{0}$, hence $|\lambda| < 1$. This reasoning holds for
any eigenvalue of $\mat{E}$, hence $\rho(\mat{E}) < 1$.
\end{proof}

\begin{lemma}
\label{lemma:cctzero}
Let $A=\wta$ be a minimal  WTA of dimension $n$ computing the strongly
convergent function $f$, and let $\vec{E}\in\R^{n^2\times n^2} $ be the Jacobian
around $\vec{s} = \sum_{t\in\Ts}\winf(t)\kron \winf(t)$ of the mapping
$F:\vec{v}\to  \T^\kron(\mat{I},\vec{v},\vec{v}) + \sum_{\sigma \in \Sigma}
\winf^\kron_\sigma$.
Then for any $c_1,c_2\in \Cs$ and any $t_1,t_2 \in \Ts$ we have $\lim_{k\to
\infty} |(\wzero(c_1)\kron \wzero(c_2))^\top \mat{E}^k
(\winf(t_1)\kron\winf(t_2))| = 0$.
\end{lemma}
\begin{proof}
Let $\ctm^\kron : \Cs \to \R^{n^2 \times n^2}$ be the context mapping
associated with the WTA $A^\kron$; i.e.\ $\ctm^\kron = \ctm_{A^\kron}$.
We start by proving by induction on $\drop{c}$ that $\ctm^\kron(c) =
\ctm(c)\kron\ctm(c)$ for all $c\in \Cs$. 
Let $\Cs^d$ denote the set of contexts $c \in \Cs$ with $\drop{c} = d$.
The statement is trivial for $c \in \Cs^0$. Assume the statement is true for all
naturals up to $d-1$ and let $c = (t,c') \in \Cs^d$ for some $t\in\Ts$ and
$c'\in\Cs^{d-1}$. Then using our inductive hypothesis we have that

\begin{align*}
\ctm^\kron(c) &=\T^\kron(\I_{n^2}, \winf(t)\kron \winf(t), \ctm(c')\kron\ctm(c'))\\
&=
\T(\I_n,\winf(t),\ctm(c'))\kron \T(\I_n,\winf(t),\ctm(c'))\\
&=
\ctm(c)\kron\ctm(c) \enspace.
\end{align*}
The case $c = (c',t)$ follows from an identical argument.

Next we use the multi-linearity of $F$ to expand $F(\vec{s} + \vec{h})$ for a
vector $\vec{h}\in \R^{n^2}$. Keeping the terms that are linear in $\vec{h}$ we
obtain that $\mat{E} = \T^\kron(\mat{I},\vec{s},\mat{I}) +
\T^\kron(\mat{I},\mat{I},\vec{s})$. 
It follows that $\mat{E} = \sum_{c\in \Cs^1} \ctm^\kron(c)$, and it can be shown
by induction on $k$ that $\mat{E}^k = \sum_{c\in \Cs^k} \ctm^\kron(c)$.

Writing $d_c = \min(\drop{c_1},\drop{c_2})$ and $d_t=\min(\depth{t_1},\depth{t_2})$,
we can see that
\begin{align*}
 \left|(\wzero(c_1)\kron \wzero(c_2))^\top \mat{E}^k
(\winf(t_1)\kron\winf(t_2))\right|
&=
\left| \sum_{c\in \Cs^k} (\wzero(c_1)\kron \wzero(c_2))^\top
\ctm^\kron(c) (\winf(t_1)\kron\winf(t_2))\right| \\
& =
\left| \sum_{c\in \Cs^k} (\wzero(c_1)^\top \ctm(c) \winf(t_1))\cdot
(\wzero(c_2)^\top \ctm(c) \winf(t_2)) \right|\\
&=
\left| \sum_{c\in \Cs^k} f(c_1[c[t_1]]) f(c_2[c[t_2]]) \right|\\
& \leq
\left(\sum_{c\in \Cs^k} |f(c_1[c[t_1]])|\right) \left(\sum_{c\in \Cs^k}
|f(c_2[c[t_2]])|\right)\\
&\leq
\left(\sum_{t\in \TF^{\geq d_c + d_t + k}} |t| |f(t)| \right)^2 \enspace,
\end{align*}
which tends to $0$ with $k \to \infty$ since $f$ is strongly convergent.
To prove the last inequality, check that any tree of the form $t' = c[c'[t]]$
satisfies $\depth{t'} \geq \drop{c} + \drop{c'} + \depth{t}$, and that for fixed
$c \in \Cs$ and $t, t' \in \Ts$ we have $|\{ c' \in \CF : c[c'[t]] = t'\}| \leq
|t'|$ (indeed, a factorization $t' = c[c'[t]]$ is fixed once the root of $t$ is
chosen in $t'$, which can be done in at most $|t'|$ different ways).
\end{proof}

\section{Proof of Theorem~\ref{gram_cvg}}
\label{proof:gram_cvg}
\begin{theorem*}%
There exists $0<\rho<1$ such that after $k$ iterations in
Algorithm~2, the approximations  $\gramcapprox$ and $\gramtapprox$ 
satisfy $\norm{\gramc- \gramcapprox}_F \leq \mathcal{O}(\rho^k)$ and  
$\norm{\gramt- \gramtapprox}_F \leq \mathcal{O}(\rho^k)$.
\end{theorem*}
\begin{proof}
The result for the Gram matrix $\gramt$ directly follows from Theorem~\ref{thm:fpoint}.
We now show how the error in the approximation of 
$\gramt = \reshape(\vec{s},n\times n)$ affects the approximation 
of $\vec{q} = (\wzero^\kron)^\top (\I - \mat{E})^{-1} = \vectorize(\gramc)$.
Let $\hat{\vec{s}}\in\R^n$ be such that 
$\norm{\vec{s} - \hat{\vec{s}}}\leq \varepsilon$,
let  $\hat{\mat{E}} = \T^\kron(\I,\hat{\vec{s}},\I)
+ \T^\kron(\I,\I,\hat{\vec{s}})$
and let $\vec{q} = (\wzero^\kron)^\top (\I - \hat{\mat{E}})^{-1}$. 
We first bound the distance between $\mat{E}$ and $\hat{\mat{E}}$. We have
\begin{align*}
\norm{\mat{E} - \hat{\mat{E}}}_F 
&=
\norm{ \T^\kron(\I,\vec{s} - \hat{\vec{s}},\I)
+ \T^\kron(\I,\I,\vec{s} - \hat{\vec{s}})}_F\\
&\leq 
2 \norm{\T^\kron}_F \norm{\vec{s} - \hat{\vec{s}}} \\
&=
\bigo{\varepsilon}
\enspace,
\end{align*}
where we used the bounds $\norm{\T(\I,\I,\v)}_F \leq \norm{\T}_F \norm{\v}$ and
$\norm{\T(\I,\v,\I)}_F \leq \norm{\T}_F \norm{\v}$.

Let $\delta = \norm{\mat{E} - \hat{\mat{E}}}$ and let $\sigma$ be the smallest 
nonzero eigenvalue of the matrix $\I - \mat{E}$. It follows from
\cite[Equation (7.2)]{invperturbation} that if $\delta < \sigma$ then
$\norm{ (\I - \mat{E})^{-1} - (\I - \hat{\mat{E}})^{-1} } \leq
\delta/(\sigma ( \sigma - \delta ))$.
Since $\delta = \mathcal{O}(\varepsilon)$ from our previous bound, the condition
$\delta \leq \sigma/2$ will be eventually satisfied as $\varepsilon \to 0$, in which
case we can conclude that
\begin{align*}
\norm{\gramc - \gramcapprox}_F
&=
\norm{ \vec{q} -\hat{\vec{q}} } \\
&\leq
\norm{ (\I - \mat{E})^{-1} - (\I - \hat{\mat{E}})^{-1} }  \norm{\wzero^\kron} \\
&\leq 
\frac{2 \delta}{\sigma^2} \norm{\wzero^\kron} \\
& = 
\mathcal{O}(\varepsilon)
\enspace.
\end{align*}
\end{proof}

\section{Proof of Theorem~\ref{prop:singv_and_SVTA_params}}
\label{proof:prop:singv_and_SVTA_params}

Let $A = \wtaSigma$ be a SVTA with $n$ states realizing a function $f$ and let 
$\singv_1 \geq \singv_2 \geq \cdots \geq \singv_n$ be the singular
values of the Hankel matrix $\mat{H}_f$.

Theorem~\ref{prop:singv_and_SVTA_params} relies on the following lemma, which explores the consequences that the fixed-point equations used to compute $\gramt$ and $\gramc$ have for an SVTA.

\begin{lemma}
\label{lemma:SVTA_singv_params}
For all $i\in [n]$, the following hold:

\begin{enumerate}
\item $\singv_i = \sum_{\sigma\in\Sigma} \winf_\sigma(i)^2+\sum_{j,k=1}^n \T(i,j,k)^2\singv_j\singv_k \enspace,$
\item $\singv_i = \wzero(j)^2 + \sum_{j,k=1}^n (\T(j,i,k)^2 + \T(j,k,i)^2)\singv_j\singv_k \enspace .$
\end{enumerate}

\end{lemma}

\begin{proof}
Let $\gramt$ and $\gramc$ be the Gram matrices associated with the rank
factorization of $\mat{H}_f$. Since $A$ is a SVTA we have $\gramt = \gramc = \mat{D}$
where $\mat{D} = \mathop{diag}(\singv_1,\cdots,\singv_n)$ is a diagonal matrix with the Hankel singular
values on the diagonal. The first equality then follows from the following fixed point 
characterization of $\gramt$:
{%
\begin{align*}
\gramt &=
\sum_{t\in \Ts} \winf(t)\winf(t)^\top \\
&=
\sum_{\sigma\in\Sigma} \winf_\sigma\winf_\sigma^\top \\
&+ 
\sum_{t_1,t_2\in\Ts} \T(\I,\winf(t_1),\winf(t_2)) \T(\I,\winf(t_1),\winf(t_2))^\top\\
&= 
\sum_{\sigma\in\Sigma} \winf_\sigma\winf_\sigma^\top + 
\mat{T}_{(1)} (\gramt\kron\gramt) \mat{T}_{(1)}^\top \enspace,
\end{align*}
}%
(where $\mat{T}_(i)$ denotes the matricization of the tensor $\T$ along the $i$th mode). 
The second equality follows from the following fixed point 
characterization of $\gramc$:
{%
\begin{align*}
\gramc
&=
\sum_{c\in \Cs} \wzero(c)\wzero(c)^\top \\
&=
\wzero\wzero^\top \\
&+ 
\sum_{c\in \Cs, t\in \Ts} \T(\wzero(c),\winf(t),\I) \T(\wzero(c),\winf(t),\I)^\top \\
&+
\sum_{c\in \Cs, t\in \Ts} \T(\wzero(c),\I,\winf(t)) \T(\wzero(c),\I,\winf(t))^\top \\
&=
\wzero\wzero^\top \\
&+
\mat{T}_{(2)} (\gramc\otimes \gramt)\mat{T}_{(2)}^\top
\\
&+
\mat{T}_{(3)} (\gramc\otimes \gramt)\mat{T}_{(3)}^\top \enspace .
\end{align*}
}%
\end{proof}

\begin{theorem*}
For any $t\in \Ts$, $c\in \Cs$ and $i,j,k\in [n]$ the following hold:
 \begin{itemize}
 \item $|\winf(t)_i| \leq \sqrt{\singv_i}$ ,
 \item $|\wzero(c)_i| \leq \sqrt{\singv_i}$ , and
 \item $|\T(i,j,k)| \leq \min\{\frac{\sqrt{\singv_i}}{\sqrt{\singv_j}\sqrt{\singv_k}},
\frac{\sqrt{\singv_j}}{\sqrt{\singv_i}\sqrt{\singv_k}},\frac{\sqrt{\singv_k}}{\sqrt{\singv_i}\sqrt{\singv_j}}\}$.
 \end{itemize}
\end{theorem*}
\begin{proof}
The third point is a direct consequence of the previous Lemma. For the first point,
let $\mat{UDV}^\top$ be the SVD of $\mat{H}_f$.
Since $A$ is a SVTA we have 
$$\winf(t)_i^2 = (\mat{D}^{1/2}\mat{V}^\top)_{i,t}^2 = \singv_i \mat{V}(t,i)^2$$
and since the rows of $\mat{V}$ are orthonormal we have $\mat{V}(t,i)^2\leq 1$.

The inequality for contexts is proved similarly by reasoning on the rows of $\mat{UD}^{1/2}$.
\end{proof}

\section{Proof of Theorem~\ref{thm:svtabound}}\label{proof:thm:svtabound}

To prove Theorem~\ref{thm:svtabound}, we will show how the computation of a
WTA on a give tree $t$ can be seen as an inner product between two tensors, one
which is a function of the topology of the tree, and one which is a function of the
labeling of its leafs (Proposition~\ref{prop:innerproduct_topo_leafs}). We will then 
show a fundamental relation between the components of the first tensor and the 
singular values of the Hankel matrix when the WTA is in SVTA normal form 
(Proposition~\ref{prop:singv_and_multicontexts}); this proposition will allow us
to show Lemma~\ref{lemma:multicontext_SVTA_truncation_bound} that bounds
the difference between components of this first tensor for the original SVTA and
its truncation. We will finally use this lemma to bound the absolute error introduced
by the truncation of an SVTA (Propositions~\ref{prop:bound_tree} and~\ref{prop:bound_sumtree}).

We first introduce another kind of contexts than the one introduced in Section~\ref{sec:approxmin},
where every leaf of a binary tree is labeled by the special symbol $*$ (which still acts
as a place holder). 
Let $\Bs$ be the set of binary trees on the one-letter alphabet $\{*\}$. 
We will call a tree $b\in \Bs$ a \emph{multicontext}. 
For any integer $M\geq 1$ we let 
$$\Bs_M =\{ b\in \Bs : |\yield{b}| = M\}$$
be the subset of multicontexts with $M$ leaves (equivalently, $\Bs_M$ is the subset
of multicontexts of size $M-1$).
Given a word $w = w_1\cdots w_M \in \Sigma^*$ and a multicontext $b\in \Bs_{M}$, 
we denote by $b[w_1,\cdots,w_M] \in \Ts_\Sigma$ the tree obtained by replacing the
$i$th occurrence of  $*$ in $b$ by $w_i$ for $i\in [M]$.
Let $b\in \Bs_M$, for any integer $m \in [M]$ we denote by 
$b\llbracket m \rrbracket\in \Bs_{M+1}$ the multicontext obtained
by replacing the $m$th occurence of $*$ in $b$ by the tree $(*,*)$.
Let $M>1$, it is easy to check that for  any $b'\in \Bs_M$, there
exist $b \in \Bs_{M-1}$ and $m \in [M-1]$ satisfying $b' = b\llbracket m \rrbracket$. 
See Figure~\ref{fig:multicontexts} for some illustrative examples.

\begin{figure}
\begin{center}
\includegraphics[scale=1]{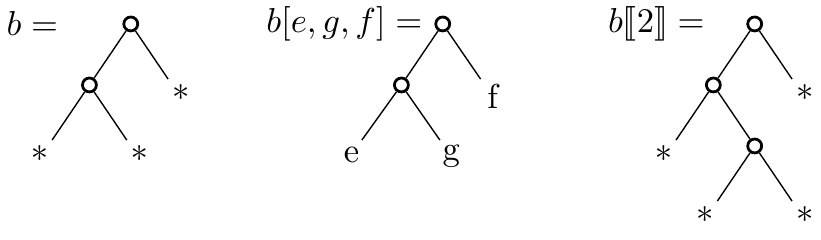}
\end{center}
\caption{A multicontext $b = ((*,*),*) \in\Bs_3$, the tree $b[e,g,f] \in \Ts$ and the
multicontext $b\llbracket 2 \rrbracket \in \Bs_4$. }
\label{fig:multicontexts}
\end{figure}

We now show how the computation of a WTA on a given tree with $M$ leaves can be
seen as an inner product between two $M$th order tensors: the first one depends
on the topology of the tree, while the second one depends on the labeling of its
leaves. Let $A = \wtaSigma$ be a WTA with $n$ states computing a function $f$. 
Given a multicontext 
$b\in \Bs_M$, we denote by $\wzeroB_A(b) \in \bigotimes_{i=1}^M \R^n$ the
$M$th order tensor inductively defined by $\wzeroB_A(*) = \wzero$ and
{
\begin{align*}
 \wzeroB_A(b\llbracket m \rrbracket)_{i_1 \cdots i_{M}} = 
\sum_{k=1}^n \wzeroB_A(b)_{i_1\cdots i_{m-1} k i_{m+2} \cdots i_{M}}
\T_{k i_{m} i_{m+1}}
\end{align*}
}%
for any $b\in \Bs_{M-1}$, $m\in [M-1]$ and $i_1, \cdots, i_M\in [n]$ 
(i.e. $\wzeroB_A(b\llbracket m \rrbracket)$ is the 
contraction of $\wzeroB_A(b)$ along the $m$th mode and $\T$ along the
first mode).
Given a word $w = w_1\cdots w_M \in \Sigma^*$, we let
$\winfB_A(w) \in \bigotimes_{i=1}^M \R^n$ be the $M$th order tensor defined
by 
{
$$\winfB_A(w)_{i_1\cdots i_M} = \winf(w_1)_{i_1} \winf(w_2)_{i_2}\cdots\winf(w_M)_{i_M}
= \prod_{m=1}^M \winf(w_m)_{i_m}$$
}%
for $i_1,\cdots, i_M \in [n]$ (i.e. $\winfB_A(w)$ is the tensor product of the $\winf(w_i)$'s). 
We will simply write $\wzeroB$ and $\winfB$ when the automaton is clear from context.

\begin{proposition}\label{prop:innerproduct_topo_leafs}
For any multicontext $b\in \Bs_M$ and 
any word $w = w_1\cdots w_M \in \Sigma^* $ we have
$$f(b[w_1,\cdots,w_M]) = \langle \wzeroB(b), \winfB(w) \rangle\enspace,$$
where the inner product between two $M$th order tensors $\ten{U}$ and $\ten{V}$
is defined by $\langle \ten{U},\ten{V}\rangle = \sum_{i_1\cdots i_M} \ten{U}(i_1,\cdots,i_M)\ten{V}(i_1,\cdots,i_M)$.
\end{proposition}
\begin{proof}[Sketch of proof]
Let $b\in \Bs_M$ and  $w = w_1\cdots w_M \in \Sigma^* $. 
Let $b_1 \in \Bs_{M-1}$ and $m\in [M-1]$ be such that
$b = b_1\llbracket m \rrbracket$.
In order to  lighten the notations and without loss of generality we assume that $m =1$.
One can check that
\begin{align*}
 \langle \wzeroB(b), \winfB(w) \rangle  =& \wzeroB(b)(\winf_{w_1}, \cdots, \winf_{w_M})\\
 =& 
 \wzeroB(b_1\llbracket 1 \rrbracket )(\winf_{w_1}, \cdots, \winf_{w_M})\\
 =&
  \wzeroB(b_1)(\winf((w_1,w_2)), \winf_{w_3}, \cdots, \winf_{w_M})\enspace .
\end{align*} 
The same reasoning can now be applied to $b_1$. Assume for example that
$b_1 = b_2 \llbracket 1 \rrbracket$ for some $b_2 \in \Bs_{M-2}$, 
we would have
\begin{align*}
 \langle \wzeroB(b), \winfB(w) \rangle  
 &=
 \wzeroB(b_1)(\winf((w_1,w_2)), \winf_{w_3}, \cdots, \winf_{w_M})\\
 &=
 \wzeroB(b_2\llbracket 1\rrbracket )(\winf((w_1,w_2)), \winf_{w_3}, \cdots, \winf_{w_M}) \\
 &=
 \wzeroB(b_2)(\winf(((w_1,w_2),w_3)), \winf_{w_4}, \cdots, \winf_{w_M})\enspace .
\end{align*}
By applying the same argument again and again we will eventually obtain
\begin{align*}
\langle \wzeroB(b), \winfB(w) \rangle  &=  \wzeroB(b_{M-1})(\winf(b[w_1,\cdots,w_M])) \\
&=  \wzeroB(*)(\winf(b[w_1,\cdots,w_M]))\\
&= \wzero^\top \winf(b[w_1,\cdots,w_M])\\
&= f(b[w_1,\cdots,w_M])\enspace. \qedhere
\end{align*}

\end{proof}

Suppose now that $A = \wtaSigma$ is an SVTA with $n$ states for $f$ and let 
$\singv_1 \geq \singv_2 \geq \cdots \geq \singv_n$ be the singular
values of the Hankel matrix $\mat{H}_f$. 
The following proposition shows a relation --- similar to the one presented in 
Theorem~\ref{prop:singv_and_SVTA_params} --- between the components of
the tensor $\wzeroB(b)$ (for any multicontext $b$) and the singular values of the Hankel matrix.

\begin{proposition}\label{prop:singv_and_multicontexts}
If $A = \wtaSigma$ is an SVTA, then for any $b\in \Bs_M$ and 
any $i_1,\cdots,i_M \in [n]$ the following holds:
$$|\wzeroB(b)_{i_1\cdots i_M}| \leq n^{M-1} \min_{p\in [M]} \{\singv_{i_p} \}
 \prod_{m=1}^M \frac{1}{\sqrt{\singv_{i_m}}} \enspace.$$
 
\end{proposition} 
\begin{proof}
We proceed by induction on $M$. If $M=1$ we have $b=*$ and
$$|\wzeroB(*)_i| = |\wzero_i| \leq \sqrt{\singv_i} = \frac{\singv_i}{\sqrt{\singv_i}}\enspace.$$
Suppose the result holds for multicontexts in $\Bs_{M-1}$ and let
$b'\in \Bs_M$. Let $m \in [M]$ and $b\in \Bs_{M-1}$ be such that 
$b' = b\llbracket m \rrbracket$. Without loss of generality and to lighten
the notations we assume that $m = 1$. Start by writing:
\begin{align*}
|\wzeroB(b')_{i_1\cdots i_M}| = |\wzeroB(b\llbracket 1 \rrbracket)_{i_1 \cdots i_M}| 
&=
\left| \sum_{k=1}^n \wzeroB_A(b)_{ k i_{3} \cdots i_{M}}
\T_{k i_1 i_2} \right| 
\leq
\sum_{k=1}^n \left| \wzeroB_A(b)_{ k i_{3} \cdots i_{M}}
\T_{k i_1 i_2} \right|
\end{align*}
Remarking that the third inequality in Theorem~\ref{prop:singv_and_SVTA_params}
can be rewritten as $|\T_{ijk}| \leq \frac{\min\{\singv_i, \singv_j,
\singv_k\}}{\sqrt{\singv_i}\sqrt{\singv_j}\sqrt{\singv_k}}$, we have for any $k
\in [n]$:
\begin{align*}
\left| \wzeroB_A(b)_{ k i_{3} \cdots i_{M}} \T_{k i_1 i_2} \right| 
\leq&
n^{M-2} \min\{\singv_k,\singv_{i_3},\cdots, \singv_{i_M}\} \frac{1}{\sqrt{\singv_k}} 
\prod_{m=3}^M \frac{1}{\sqrt{\singv_{i_m}}} 
\frac{\min\{\singv_k, \singv_{i_1}, \singv_{i_2}\}}{\sqrt{\singv_k}\sqrt{\singv_{i_1}}\sqrt{\singv_{i_2}}} \\
=&
n^{M-2} \frac{1}{\singv_k}  
\prod_{m=1}^M \frac{1}{\sqrt{\singv_{i_m}}} 
\min\{\singv_k,\singv_{i_3},\cdots, \singv_{i_M}\} \min\{\singv_k, \singv_{i_1}, \singv_{i_2}\}  \\
\leq &
n^{M-2} \min_{p\in [M]} \{\singv_{i_p} \}
 \prod_{m=1}^M \frac{1}{\sqrt{\singv_{i_m}}} \enspace,
\end{align*}
where we used that
$$\min\{\singv_k,\singv_{i_3},\cdots, \singv_{i_M}\} \min\{\singv_k, \singv_{i_1}, \singv_{i_2}\} \leq \singv_k \min\{\singv_{i_1},\cdots, \singv_{i_M}\}$$
Summing over $k$ yields the desired bound.
\qedhere
\end{proof}

Let $\hat{f}$ be the function computed
by the SVTA truncation of $A$ to $\hat{n}$ states. 
Let $\mat{\Pi} \in \R^{n\times n}$ be the diagonal matrix defined by
$\mat{\Pi}(i,i) = 1$ if $i \leq \hat{n}$ and $0$ otherwise. It is easy to 
check that the WTA 
$\hat{A} = \langle\hat{\wzero},\hat{\T},\hat{\winf}_\sigma\rangle$,
where $ \hat{\wzero} =  \mat{\Pi}\wzero$, $\hat{\T}=\T(\I, \mat{\Pi}, \mat{\Pi})$ and
$\hat{\winf}_\sigma =\winf_\sigma$,
computes the function $\hat{f}$. We let $\hat{\winf}(t) = \winf_{\hat{A}}(t)$
for any tree $t$ and similarly for $\hat{\wzero}(c)$, $\hat{\winfB}(w)$ and $\hat{\wzeroB}(c)$.  

We can now prove the following Lemma that bounds the absolute difference
between the components of the tensors $\wzeroB(b)$ and $\hat{\wzeroB}(b)$ for
a given multicontext $b$.

\begin{lemma}\label{lemma:multicontext_SVTA_truncation_bound}
For any $b\in \Bs_M$ and any $i_1,\cdots i_M \in [n]$ we have 
$$|(\wzeroB(b) - \hat{\wzeroB}(b))_{i_1\cdots i_M} | 
\leq \singv_{\hat{n} + 1}\ n^{M-1}  \prod_{m=1}^M \frac{1}{\sqrt{\singv_{i_m}}} \enspace .$$
\end{lemma} 
\begin{proof}
It is easy to check that when there exists at least one $m\in [M]$ such that
$i_m > \hat{n}$, we have $\hat{\wzeroB}(b)_{i_1\cdots i_M} = 0$, 
hence
$$|(\wzeroB(b) - \hat{\wzeroB}(b))_{i_1\cdots i_M} | = |\wzeroB(b)_{i_1\cdots i_M} | $$
and the result directly follows from Proposition~\ref{prop:singv_and_multicontexts}.

Suppose $i_1,\cdots, i_M \in [\hat{n}]$,
we proceed by induction on $M$. If $M=1$ then $b=*$, thus
$$|\wzeroB(*)_i - \hat{\wzeroB}(*)_i| = |\wzero_i - \hat{\wzero}_i| = 0$$
for all $i \in [\hat{n}]$.

Suppose the result holds for multicontexts in $\Bs_{M-1}$ and let
$b' \in \Bs_{M}$. Let $b\in \Bs_{M-1}$ and $m \in [M-1]$ be  such that
$b' = b\llbracket m \rrbracket$. To lighten the notations we assume
without loss of generality that $m = 1$. We have
\begin{align}
|(\wzeroB(b') - \hat{\wzeroB}(b'))_{i_1\cdots i_{M}} | 
\label{eq:multicontext_1}
=&
|(\wzeroB(b\llbracket 1 \rrbracket) - \hat{\wzeroB}(b\llbracket 1 \rrbracket))_{i_1\cdots i_{M}} |\\
\leq&
\label{eq:multicontext_2.1}
\sum_{k=1}^{\hat{n}} \left| \T_{k i_{1} i_2} \right| 
| (\wzeroB(b)  -  \hat{\wzeroB}(b))_{k i_{3}\cdots i_{M}} | \\
+& 
\label{eq:multicontext_2.2}
\hspace{-0.2cm}\sum_{k=\hat{n}+1}^{n} \left| \T_{k i_{1} i_2} \right| \left| \wzeroB(b)_{k i_{3}\cdots i_{M}} \right| \\
\leq &
\label{eq:multicontext_3.1}
\sum_{k=1}^{\hat{n}} \frac{\sqrt{\singv_k}}{\sqrt{\singv_{i_1}\singv_{i_2}}} 
\cdot \frac{\singv_{\hat{n} + 1}\ n^{M-2} }{\sqrt{ \singv_k} \sqrt{\singv_{i_3}} \cdots \sqrt{\singv_{i_M}}} \\
+ & 
\label{eq:multicontext_3.2}
\hspace{-0.2cm}\sum_{k=\hat{n}+1}^{n} \frac{\sqrt{\singv_k}}{\sqrt{\singv_{i_1}\singv_{i_2}}} 
\cdot \frac{\min\{\singv_k, \singv_{i_3}, \cdots, \singv_{i_M}\}\ n^{M-2} }
				 {\sqrt{ \singv_k} \sqrt{\singv_{i_3}} \cdots \sqrt{\singv_{i_M}}} \\
\leq &
\ \singv_{\hat{n}+1}\ n^{M-1}  \prod_{m=1}^{M} \frac{1}{\sqrt{\singv_{i_m}}} \enspace .
\end{align}
To decompose~(\ref{eq:multicontext_1}) in~(\ref{eq:multicontext_2.1}) and~(\ref{eq:multicontext_2.2})
we used the fact that $\T_{k i_1 i_2} = \hat{\T}_{k i_1 i_2}$ whenever $k \leq \hat{n}$ and
$\hat{\wzeroB}(b)_{k i_3\cdots i_M} = 0$ whenever $k > \hat{n}$. We bounded~(\ref{eq:multicontext_2.1}) 
by~(\ref{eq:multicontext_3.1}) using the induction hypothesis, while we used Proposition~\ref{prop:singv_and_multicontexts}
to bound~(\ref{eq:multicontext_2.2}) by~(\ref{eq:multicontext_3.2}).  
\end{proof}

\begin{proposition}\label{prop:bound_tree}
Let $t\in \Ts$ be a tree of size $M$, then 
\begin{align*}
|f(t) - \hat{f}(t)| 
&\leq 
n^{2M-1} \singv_{\hat{n}+1} 
\enspace .
\end{align*}
\end{proposition}
\begin{proof}
Let $t \in \Ts$ be a tree of size $M-1$, then there exists a (unique) $b \in \Bs_{M}$ 
and a (unique) word $w = w_1 \cdots w_{M} \in \Sigma^*$ 
such that $t = b[w_1,\cdots, w_{M} ]$. Since $\winf_\sigma = \hat{\winf}_\sigma$ for
all $\sigma \in \Sigma$, we have $\winfB(x) = \hat{\winfB}(x)$ for all $x\in \Sigma^*$. 
Furthermore, since $\winf_\sigma(i)^2 \leq \singv_i$ for all $i\in [n]$, we have
$$|\winfB(w)_{i_1\cdots i_M}| \leq \prod_{m=1}^M \sqrt{\singv_{i_m}}\enspace.$$  
It follows that 
\begin{align*}
|f(t) - \hat{f}(t)| 
= & 
\left|\langle \wzeroB(b),\winfB(w) \rangle - \langle \hat{\wzeroB}(b),\hat{\winfB}(w) \rangle \right|\\
= &
\left|\langle \wzeroB(b) - \hat{\wzeroB}(b),\winfB(w) \rangle \right|\\
\leq &
\sum_{i_1=1}^n \cdots \sum_{i_M=1}^n |(\wzeroB(b) - \hat{\wzeroB}(b))_{i_1\cdots i_M} |\ | \winfB(w)_{i_1\cdots i_M} |\\
\leq &
\sum_{i_1=1}^n \cdots \sum_{i_M=1}^n \singv_{\hat{n} + 1}\ n^{M-1}  \prod_{m=1}^M \frac{1}{\sqrt{\singv_{i_m}}}
\cdot \prod_{m=1}^M \sqrt{\singv_{i_m}} \\
= &
n^{2M - 1} \singv_{\hat{n} +1}
\end{align*}
\end{proof}

\begin{proposition}\label{prop:bound_sumtree}
Let $S = |\Sigma|$ be the
size of the alphabet. For any integer $M$ we have
$$ \sum_{t\in\Ts:\atop\size{t} < M}  |f(t) - \hat{f}(t)| \leq  \frac{(4Sn^2)^{M+1}-1}{(4Sn^2)-1} \singv_{\hat{n}+1}\enspace.$$
\end{proposition}
\begin{proof}
For any integer $m$ there are less than $4^m$ binary trees with $m$ internal nodes (which is a bound
on the $m$-th Catalan number) and each one of these trees has $m+1$ leaves, thus $S^{m+1}$ possible
labelling of the leaves. Using the previous proposition we get
\begin{align*}
\sum_{t\in\Ts:\atop\size{t} < M}  |f(t) - \hat{f}(t)| 
 &=
 \sum_{m = 0}^{M-1}   \sum_{t\in\Ts:\atop\size{t} = m}  |f(t) - \hat{f}(t)| \\
&\leq
\sum_{m = 0}^{M-1} 4^m S^{m+1} \cdot  n^{2(m-1)}  \singv_{\hat{n}+1}\\
&\leq
\sum_{m = 1}^{M} (4Sn^2)^{m}  \singv_{\hat{n}+1} \\
&\leq
 \frac{(4Sn^2)^{M+1}-1}{(4Sn^2)-1}  \singv_{\hat{n}+1}.\qedhere
\end{align*}
\end{proof}

\begin{theorem*}
Let $A = \wtaSigma$ be a SVTA with $n$ states realizing a function $f$ and let 
$\singv_1 \geq \singv_2 \geq \cdots \geq \singv_n$ be the singular
values of the Hankel matrix $\mat{H}_f$. Let $\hat{f}$ be the function computed
by the SVTA truncation of $A$ to $\hat{n}$ states.

Let $S=|\Sigma|$ be the size of the alphabet, let $M$ be an integer and let $\varepsilon > 0$.

\begin{itemize}
\item  For any tree $t\in \Ts$ of size $M$,
 if ${\displaystyle M < \frac{\log\left(\frac{1}{\singv_{\hat{n}+1}}\right) + \log\left(\varepsilon\right) }{2\log n} }$
then $|f(t) - \hat{f}(t)| < \varepsilon$.
\item If ${\displaystyle M < \frac{\log\left(\frac{1}{\singv_{\hat{n}+1}}\right) + \log(\varepsilon)}{\log(4Sn^2)} } - 1$ then
$\displaystyle\sum_{t:\size{t}<M} |f(t)-\hat{f}(t)|< \varepsilon$.
\end{itemize}

\end{theorem*}
\begin{proof}
For the first bound, it is easy to check that if 
$$M < \frac{\log\left(\frac{1}{\singv_{\hat{n}+1}}\right) + \log\left(\varepsilon\right) }{2\log n}$$
then $ n^{2M} \singv_{\hat{n}+1} < \varepsilon$ and the result follows from 
Proposition~\ref{prop:bound_tree}. For the second one, if
$$ M < \frac{\log\left(\frac{1}{\singv_{\hat{n}+1}}\right) + \log(\varepsilon)}{\log(4Sn^2)} - 1$$ then
 $\frac{(4Sn^2)^{M+1}-1}{(4Sn^2)-1}\singv_{\hat{n}+1} < \varepsilon$ and the
 result follows from Proposition~\ref{prop:bound_sumtree}.
\end{proof}

\newpage
\clearpage

\end{document}